\documentclass[pdflatex,sn-mathphys-num]{sn-jnl}

\usepackage{graphicx}%
\usepackage{multirow}%
\usepackage{amsmath,amssymb,amsfonts}%
\usepackage{amsthm}%
\usepackage{bbm}
\usepackage{mathrsfs}%
\usepackage[title]{appendix}%
\usepackage{xcolor}%
\usepackage{textcomp}%
\usepackage{manyfoot}%
\usepackage{booktabs}%
\usepackage{algorithm}%
\usepackage{algorithmicx}%
\usepackage{algpseudocode}%
\usepackage{listings}%
\usepackage[numbers]{natbib}

\theoremstyle{thmstyleone}%
\newtheorem{theorem}{Theorem}
\newtheorem{proposition}[theorem]{Proposition}%

\theoremstyle{thmstyletwo}%
\newtheorem{remark}{Remark}%

\newtheorem{definition}{Definition}
\newtheorem{lemma}{Lemma}

\newtheorem{thm}{Theorem}
\newtheorem{corollary}{Corollary}

\theoremstyle{thmstylethree}%

\ifodd 0

\newcommand{\com}[1]{{\color{red}\textbf{Parinaz's Comment}: #1}}
\newcommand{\comy}[1]{{\color{orange}\textbf{Yifan's Comment}: #1}}
\newcommand{\resp}[1]{{\color{cyan}\textbf{Response}: #1}} 
\else

\newcommand{\com}[1]{}
\newcommand{\comy}[1]{}
\newcommand{\resp}[1]{}
\fi

\begin{document}

\title[Article Title]{Generalization Error Bounds for Learning under Censored Feedback}


\author*[1]{\fnm{Yifan} \sur{Yang}}\email{yang.5483@osu.edu}

\author[2]{\fnm{Ali} \sur{Payani}}\email{apayani@cisco.com}

\author[3]{\fnm{Parinaz} \sur{Naghizadeh}}\email{pnaghizadeh@ucsd.edu}

\affil*[1]{\orgname{The Ohio State University}, \orgaddress{\city{Columbus}, \state{OH}, \country{USA}}}

\affil[2]{\orgname{Cisco Research}, \orgaddress{ \city{San Jose},  \state{CA}, \country{USA}}}

\affil[3]{\orgname{University of California San Diego}, \orgaddress{\city{San Diego}, \state{CA}, \country{USA}}}


\abstract{Generalization error bounds from learning theory provide statistical guarantees on how well an algorithm will perform on previously unseen data. In this paper, we characterize the impacts of data non-IIDness due to \emph{censored feedback} (a.k.a. selective labeling bias) on such bounds. Censored feedback is ubiquitous in many real-world online selection and classification tasks (e.g., hiring, lending, recommendation systems) where the true label of a data point is only revealed if a favorable decision is made (e.g., accepting a candidate, approving a loan, displaying an ad), and remains unknown otherwise. We first derive an extension of the well-known Dvoretzky-Kiefer-Wolfowitz (DKW) inequality, which characterizes the gap between empirical and theoretical data distribution CDFs learned from \emph{IID} data, to problems with \emph{non-IID data due to censored feedback}. We then use this CDF error bound to provide a bound on the generalization error guarantees of a classifier trained on such non-IID data. We show that existing generalization error bounds (which do not account for censored feedback) fail to correctly capture the model's generalization guarantees, verifying the need for our bounds. We further analyze the effectiveness of (pure and bounded) exploration techniques, proposed by recent literature as a way to alleviate censored feedback, on improving our error bounds. Together, our findings illustrate how a decision maker should account for the trade-off between strengthening the generalization guarantees of an algorithm and the costs incurred in data collection when future data availability is limited by censored feedback. }



\maketitle

\section{Introduction}\label{sec:intro}
Generalization error bounds are a fundamental concept in machine learning theory. They provide (statistical) guarantees on how well a machine learning algorithm trained on some given dataset will perform on new, unseen data. In order to derive such theoretical guarantees for the performance of ML models, existing works often make (implicit or explicit) assumptions about the training data. These assumptions include access to independent and identically distributed (IID) training data, the availability of correct labels, and static underlying data distributions \citep{bartlett2002rademacher,bousquet2002stability,cortes2019region,cortes2020adaptive}. Some studies in this area, e.g. \citet{mohri2007stability,mohri2008rademacher,kuznetsov2017generalization,cheng2018simple}, have provided bounds when these assumptions are removed. In this paper, we are similarly interested in the impact of non-standard assumptions on the training data, specifically due to \emph{censored feedback}, on the learned algorithm's generalization error guarantees. 

Censored feedback, also known as {selective labeling bias}, arises in many applications wherein human or algorithmic decision-makers set certain thresholds or criteria for favorably classifying individuals, and subsequently only observe the true label of individuals who pass these requirements. For example, schools may require a minimum GPA or standardized exam score for admission; yet, graduation rates are only observed for admitted students. Financial institutions may set limits on the minimum credit score required for loan approval; yet, loan return rates are only observed for approved applicants. In these types of classification tasks, the algorithm's training dataset grows over time (as students are admitted, loans are granted); however, the new data is selected in a non-IID manner from the underlying domain, due to the unobservability of the true label of rejected data. This type of bias also arises when determining recidivism in courts, evaluating the effectiveness of medical treatments, assessing the interest of a user in an ad or content,  flagging fraudulent online credit card transactions, etc. Despite this ubiquity, to the best of our knowledge, generalization error bounds given non-IID training data due to censored feedback remain unexplored. \emph{We close this gap by providing such bounds in this work, show the need for them, and formally establish the extent to which censored feedback hinders generalization.} 

We further explore the effects of a common strategy for handling censored feedback on improving learning guarantees. Specifically, one of the commonly proposed methods to alleviate the impacts of censored feedback is to \emph{explore} the data domain, and admit (some of) the data points that would otherwise be rejected, with the goal of expanding the training data. Existing approaches to exploration can be categorized into \emph{pure exploration} \citep{nie2018adaptively,bechavod2019equal,kazerouni2020active,kilbertus2020fair}, where any individual in the exploration range may be admitted (with some probability $\epsilon$), and \emph{bounded exploration} \citep{balcan2007margin, wei2021decision,yang2022adaptive, lee2023partitioned}, in which the exploration range is further limited based on cost or informativeness of the new samples. The additional data samples collected through (pure or bounded) exploration may not only help improve the accuracy of the learned model when evaluated on a given test data (as shown by these prior works), but may also help tighten the generalization error guarantees of the learned model; we formalize the latter improvement, and show how the frequency and range of exploration can be adjusted accordingly. 

We note that censored feedback may or may not be avoidable depending on the application (given, e.g., the costs or legal implications of exploration). We therefore present generalization error bounds both with and without exploration, establishing the extent to which the decision maker should be concerned about censored feedback's impact on the learned model's guarantees, and how well they might be able to alleviate it if exploration is feasible. 

\paragraph{Our approach} We characterize the generalization error bounds as a function of the gap between the empirically estimated cumulative distribution function (CDF) obtained from the training data, and the ground truth underlying distribution of data. At the core of our approach is noting that although censored feedback leads to training data being sampled in a non-IID fashion from the true underlying distribution, this non-IID data can be split into IID ``subdomains''. Existing error bounds for IID data, notably the Dvoretzky-Kiefer-Wolfowitz (DKW) inequality~\citep{dvoretzky1956asymptotic,massart1990tight}, can provide bounds on the deviation of the empirical and theoretical subdomain CDFs, as a function of the number of available data samples in each subdomain. The challenge, however, lies in reassembling such subdomain bounds into an error bound on the full domain CDFs. Specifically, this will require us to shift and/or scale the subdomain CDFs, with shifting and scaling factors that are themselves empirically estimated from the underlying data, and can be potentially re-estimated as more data is collected. Our analysis identifies these factors, and highlights the impacts of each on the error bounds.

\paragraph{Summary of findings and contributions} 
\begin{enumerate}
\item We generalize the well-known DKW inequality, which characterizes the gap between empirical and theoretical CDFs given \emph{IID} data, to problems with \emph{non-IID data due to censored feedback} without exploration (Theorem~\ref{thm:two_subdomains}) and with exploration (Theorem~\ref{thm:three_subdomains}), and formally show the extent to which censored feedback hinders generalization. 

\item We characterize the change in these error bounds as a function of the severity of censored feedback (Proposition~\ref{prop:B(LB) < B(theta)}) and the exploration frequency (Proposition~\ref{prop:B(LB, theta, 1) < B(LB, theta, epsilon)}). We further show (Section~\ref{subsec: illustration}) that a minimum level of exploration is needed to tighten the error bound.   

\item We derive a generalization error bound (Theorem~\ref{thm:error}) for a classification model learned in the presence of censored feedback using the CDF error bounds in Theorems~\ref{thm:two_subdomains} and~\ref{thm:three_subdomains}. 

\item We numerically illustrate our findings (Section~\ref{sec:numerical}). We show that existing generalization error bounds (which do not account for censored feedback) fail to correctly capture the generalization error guarantees of the learned models. We also illustrate how a decision maker should account for the trade-off between strengthening the generalization guarantees of an algorithm and the costs incurred in data collection for reaching enhanced learning guarantees. 
\end{enumerate}
 
\section{Related Works}\label{sec:related}

Although existing literature have studied generalization error bounds for learning from non-IID data, non-IIDness raised by censored feedback has been overlooked. Here, we discuss works most closely related to ours. We also provide a more detailed review of other related work in Appendix~\ref{app:additional_review}.

First, our work is closely related to generalization theory in the PAC learning framework in non-IID settings, including works such as \citep{ mohri2007stability, mohri2008rademacher, kuznetsov2017generalization, yu1994rates}; These works consider dependent samples generated through a stationary or and non-stationary $\beta$-mixing sequence, respectively, where the dependence between samples weakens over time. To address the vanishing dependence issue, these works identify building blocks within which the data can be treated as IID. The study of \citet{yu1994rates} is based on the VC-dimension, while \citet{mohri2008rademacher} and \citet{mohri2007stability} focus on the Rademacher complexity and algorithm stability, respectively. \citet{kuznetsov2017generalization} further extends this analysis to nonstationary mixing sequences. Our work shares a similar high-level approach: we also construct IID blocks to address data non-IIDness. However, we differ in key aspects: the method of reassembling these blocks, the underlying source of data non-IIDness, and our explicit consideration of exploration's impact. Conceptually, our generalization bounds are also distinct. While prior work derives bounds using the mixing parameter $\beta$, treating samples across blocks as IID, our bounds emerge from a threshold-based data collection mechanism. They are constructed by reassembling multiple IID blocks while explicitly accounting for the effects of censored feedback.

Our work is also closely related to partitioned active learning, including \citet{cortes2019region, cortes2020adaptive, lee2023partitioned, zheng2019generalization}. \citet{cortes2019region} partition the entire domain to find the best hypothesis for each subdomain, and a PAC-style generalization bound is derived compared to the best hypothesis over the entire domain. This idea is further extended to adaptive partitioning in \citet{cortes2020adaptive}. In \citet{lee2023partitioned}, the domain is partitioned into a fixed number of subdomains, and the most uncertain subdomain is explored to improve the mean-squared error. The work of \citet{zheng2019generalization} considers a special data non-IIDness where the data-generating process depends on the task property, partitions the domain according to the task types, and analyzes each subdomain separately. Our work is similar to these studies in that we also consider (active) exploration techniques, and partition the data domain to build IID blocks. However, we differ in problem setup and analysis approach, and in accounting for the cost of exploration when we consider bounded exploration techniques. More specifically, their bounds are derived from the aggregation of multiple subdomains with requested labels (we refer to as exploration). The key distinction lies in data availability: while they can request labels from any subdomain without considering the cost of doing so, we are constrained to explore samples from certain subdomains due to the presence of censored feedback.

The technique of identifying IID-blocks within non-IID datasets has also been used in other (application) contexts to address the challenge of generalization guarantees given non-IID data. For instance, \citet{wang2023cross} investigate generalization performance with covariate shift and spatial autocorrelation in geostatistical learning. They address the non-IIDness issue by removing samples from the buffer zone to construct spatially independent folds. Similarly, \citet{tang2021personalized} study generalization performance within the Federated Learning paradigm with non-IID data. They employ clustering techniques to partition clients into distinct clusters based on statistical characteristics, thus treating samples from clients within each cluster as IID and analyzing each cluster separately. We similarly explore generalization performance with non-IID data samples and employ the technique of identifying IID subdomains/blocks. However, we differ in the reason for the occurrence of non-IIDness, the setup of the problem, and our analytical approaches.

We are further related to the literature on active learning with censored feedback. \citet{leme2023pricing} investigate the query complexity in auction settings, leveraging the notion of relative flatness to efficiently explore the optimal reserved price in the right region of the distribution without fully recovering the entire value distribution. \citet{zhang2023learning} actively learn the threshold under censored feedback, deriving PAC-style query complexity bounds under monotonicity and Lipschitz conditions, and extend their results to adversarial online settings using bandit-based algorithms. Like these works, we study learning under threshold-based censoring of the collected data, but differ in the threshold setting, the setup of the problem, and the impact of exploration on the generalization error bounds. 

Lastly, our work is also connected to broader literature on learning under censored feedback. The multi-armed bandit learning \citep{bubeck2012regret,lattimore2020bandit} explores the exploration-exploitation trade-off, which similarly arises in online/machine learning scenarios where collecting additional data can improve generalization but incurs cost. \citet{chien2023algorithmic} examine harms from censored feedback and propose randomized exploration (and recourse) strategies to mitigate the impact of censored feedback. \citet{guinet2022bandit} studies sequential decision-making in bandit settings under strategic and random censored feedback, analyzing regret through the effective dimension and a generalized Elliptical Potential Inequality. \citet{abernethy2016threshold} investigate the threshold bandit problem with binary rewards, proposing UCB-style algorithms with regret guarantees based on DKW and Kaplan-Meier estimators under uncensored and censored feedback. \citet{lugosi2022hardness} analyze a repeated newsvendor problem under censored demand using regret-minimizing algorithms that combine tailored cost estimators with randomized exploration to achieve optimal performance even in adversarial, nonstationary settings. The key difference in our approach is that we consider \emph{bounded} exploration (motivated by works such as \citep{balcan2007margin,lee2023partitioned, wei2021decision, yang2022adaptive, Yang2025}), where a bound is set to limit the ``arms'' that are considered for exploration; here, ranges of data samples that may be admitted). This is due to the assumption that the cost of wrong decisions increases as samples further away from the current decision threshold are admitted, making some arms too costly for exploration. Furthermore, in the bandit literature, regret analysis is conducted to analyze the model performance compared to the best actions in hindsight. In contrast, we analyze the model performance from a different angle: our goal is to improve the generalization error guarantees (upper bound on the difference between the model's performance on training data and unseen testing data) by utilizing the newly collected samples through exploration.

\section{Problem Setting}\label{sec:model}

We consider a supervised learning setup where a learner (equivalently, the learning algorithm) selects a classifier based on an initial training dataset and subsequently uses it to make binary decisions (e.g., accept/reject) for new data samples arriving sequentially. We use a bank granting loans as a running example. 

\vspace{0.1in}
\textbf{The data.} Each data sample is represented as a pair $(x,y)$, where $x\in\mathcal{X}\subseteq \mathbb{R}$ is the feature used for decision-making (e.g., credit score)\footnote{Extensions to high-dimensional samples are discussed in Appendix~\ref{app:higher_dimension}.}, and $y\in \mathcal{Y}=\{0,1\}$ is the true label indicating qualification status, with $y=1$ denoting that the sample is qualified to receive a favorable decision (e.g., the applicant will repay the loan if granted). We denote the corresponding random variables as $X$ and $Y$, and we use $F^y(x) = \mathbb{P}(X \leq x|Y=y)$ to denote the cumulative distribution function (CDF) of $X$ conditional on $Y=y$, and $p_y = \mathbb{P}(Y=y)$ to denote the label rates in the population. 

\vspace{0.1in}
\textbf{The learning algorithm.} The learner begins with an initial training dataset consisting of $n_y$ IID samples\footnote{We assume that any non-IIDness is introduced due to censored feedback impacting subsequent data collection. Extension to initially biased training data is possible but at the expense of additional notation.} $\{x^y_i\}_{i=1}^{n_y}$ for each label $y \in \{0,1\}$. Based on the initial dataset, the learner selects a threshold-based binary classifier $f_{\theta}(x):\mathcal{X}\rightarrow\{0,1\}$ (i.e., $f_\theta(x) = \mathbbm{1}(x\geq \theta)$) to decide whether to accept or reject (equivalently, assign labels 1 or 0) incoming applications, where $\theta$ denotes the decision threshold (e.g., $\theta$ could be the minimum credit score to be approved for a loan).\footnote{The threshold classifier assumptions is not too restrictive: \citet[Thm 3.2]{corbett2017algorithmic} and \citet{raab2021unintended} have shown that threshold classifiers can be optimal if multi-dimensional features can be appropriately converted into a one-dimensional scalar (e.g., with a neural network).} 

\vspace{0.1in}
\textbf{Censored vs. disclosed regions.} The decision threshold $\theta$ divides the data domain into two regions: the upper, \emph{disclosed} region, where the true label of future admitted samples will become known to the learner, and the lower, \emph{censored} region, where true labels are no longer observed. As new samples arrive, due to this censored feedback, additional data is only collected from the disclosed region of the data domain (e.g., we only find out if an individual repays the loan if it is granted the loan in the first place). This is what causes the non-IIDness of the (expanded) dataset: after new samples arrive, the training dataset consists of $n_y$ initial IID samples from both censored and disclosed regions on each label $y$, and an additional $k_y$ samples collected afterward from each label $y$, but only from the disclosed region, making the entire $n_y+k_y$ samples a non-IID subset of the respective label $y$'s data domain. 

\begin{remark}\label{remark:prior-posterior} We note that there are two possible ways to interpret the additional samples $\{k_y\}_{y\in\{0,1\}}$: \emph{a posteriori} (i.e., outcomes after collecting exactly $k_y$ new samples in the disclosed region), or \emph{a priori} (i.e., possible values once a total of $T$ new samples arrive, only some of which will fall in the disclosed region). The former is a reasonable assumption if a learner has already collected samples under censored feedback, or alternatively, is willing to wait to collect the exact required number of samples until it can achieve a desired error bound. The latter is from the viewpoint of a learner contemplating potential outcomes if it waits for a total of $T$ new samples to arrive. We will present our new error bound under both interpretations.  
\end{remark}

\textbf{Estimating the underlying CDF of the data.} Formally, let $F^y(x)$ denote the theoretical (ground truth) CDF for label $y$ samples. Let $\alpha^y:={F}^y(\theta)$ be the theoretical fraction of samples in the censored region, and $m_y$ be the random number of the initial $n_y$ training samples from label $y$ samples located in the censored region. It is worth noting that $\frac{m_y}{n_y}$ can provide an empirical estimate of $\alpha^y$, but the two are in general not equal. After new samples are collected, the learner has access to $n_y+k_y$ total samples from label $y$ samples, which are not identically distributed: $m_y$ are in the censored region, and $n_y-m_y+k_y$ are in the disclosed region. Let $F^y_{n_y+k_y}(x)$ denote the empirical CDF of the feature distribution for label $y$ samples based on these $n_y+k_y$ training data points. Our first goal is to provide an error bound, similar to the DKW inequality, of the discrepancy between $F^y_{n_y+k_y}(x)$ and the ground truth CDF $F^y(x)$, for each label $y$. We will then use these to bound the generalization error guarantees of the learned model from the (non-IID) $\{n_y+k_y\}_{y\in \{0,1\}}$ data points.

\vspace{0.1in}
\textbf{\emph{When} will an error bound be found? Two viewpoints.} In the described setup, the threshold-based binary classifier $f_\theta(x)$ is selected based on a collection of initial training data, and then the resulting threshold $\theta$ is adopted and impacts future data collection. We consider two viewpoints of \emph{when} in this process a CDF/generalization error bound is found:
\begin{itemize}
    \item An \emph{ex-post analysis}, where the initial training data is an already collected \emph{realization} of $\{n_y\}_{y\in\{0,1\}}$ data points drawn from the underlying data distributions (shown in red in Fig.~\ref{fig:randomness}). This results in a realization (fixed) $\theta$, with future data drawn randomly from the underlying data distributions and selected or rejected according to this threshold. In our running example, this is akin to saying that a bank has some historical training data, which it will use to set a threshold to approve future loans, which will in turn induce censored feedback in its future data. 
    \item An \emph{ex-ante analysis}, where the initial training samples have not been collected yet, and may be any of the possible realizations of the underlying data distribution (shown in green in Fig.~\ref{fig:randomness}), which leads to a random threshold $\theta$. In our running example, this is akin to saying that a bank has no data, so it will decide to collect an initial set of $\{n_y\}_{y\in\{0,1\}}$ data points at random, and then use those to set the threshold and collect additional (censored) data.  
\end{itemize} 

We first discuss our findings under the \emph{ex-post analysis} case, and then present the \emph{ex-ante analysis} case in Section~\ref{sec:full_random}.  

\begin{figure}[ht]
	\centering
    \includegraphics[width=0.75\textwidth]{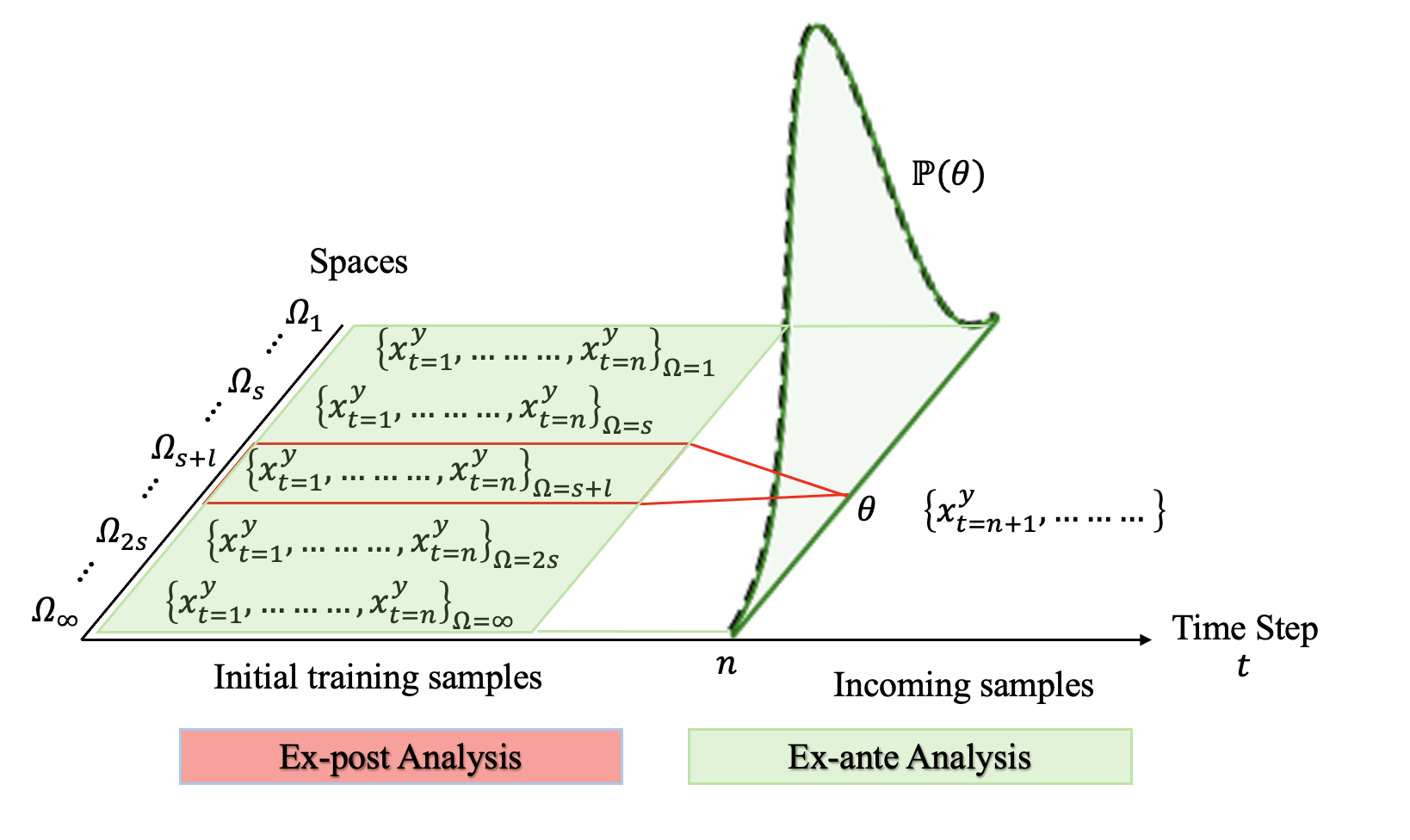}
    \caption{Illustration on the difference between the \emph{ex-post analysis} and the \emph{ex-ante analysis} cases.}
	\label{fig:randomness}
\end{figure}

\vspace{0.1in}
\textbf{Summary of problem setting.} Our problem setting is summarized below. We also summarize all notations in Table~\ref{tab1:notation}. 

\vspace{0.05in}
\emph{$\bullet$ Stage I: Initial Training Samples and Threshold Selection.} The learner's initial training data consists of $n_y$ data points, $\{x^y_i\}_{i=1}^{n_y}$, from each label $y\in\{0,1\}$ (either collected in advance, or to be collected). These data points are drawn IID from the corresponding true underlying distributions with CDF $F^y(x)$. Accordingly, the learner selects a decision threshold $\theta$. Once $\theta$ is found, the initial $n_y$ samples with label $y$ are divided into $m_y$ samples below $\theta$ ($m_y=|\{i: x^y_i< \theta\}|$, referred to as the censored region) and $n_y-m_y$ samples above $\theta$ (referred to as the disclosed region). 

\vspace{0.05in}
\emph{$\bullet$ Stage II: Arrival of New Data Samples.} At each time $t$, a new sample $(\hat{x},\hat{y})$ arrives. Its true label is $\hat{y}=y$ with probability $p_y$, and its feature $\hat{x}$ is drawn uniformly at random from the corresponding conditional distribution with CDF $F^{\hat{y}}(x)$. The sample's feature $\hat{x}$ is observed, and it is admitted if and only if $\hat{x}\geq \theta$. Due to censored feedback, $\hat{y}$ will only be observed if the sample is admitted. When a sample is admitted, its data is used to expand the corresponding dataset of $y=\hat{y}$ samples to $\{x^y_1, \ldots, x^y_{n_y}, x^y_{n_y+1}, \ldots, x^y_{n_y+k^t_y-1}, \hat{x}\}$.

\vspace{0.05in}
\emph{Finding Error Bounds.} Stage II is repeated $T$ times (which can be fixed in advance, or can be the time at which a certain number of samples have been collected). At this time, the learner has access to $k_y$ new samples for each label $y$. These samples have expanded the initial (IID) training dataset for label $y$ into a now non-IID collection $\{x^y_1, \ldots, x^y_{n_y}, x^y_{n_y+1}, \ldots, x^y_{n_y+k_y-1}, x^y_{n_y+k_y}\}$. The learner finds $F^y_{n_y+k_y}(x)$, the empirical CDF of the feature distribution for label $y$, based on the combined {$n_y+k_y$} data points. The \emph{ex-post analysis} sets a bound on the error in this CDF estimate after Stage I. The \emph{ex-ante analysis} sets a bound on the error in this CDF estimate before Stage I.  

\begin{table}
\caption{Notation Summary}\label{tab1:notation}
\begin{tabular}{@{}lp{10cm}@{}}
\toprule
\textbf{Symbol} & \textbf{Explanation} \\
\midrule
$(x,y)$ & Paired (feature, label) information of samples\\
  $\theta$   & Decision threshold\\
  $LB$   & Exploration lower bound \\
  $n ~ (n_y)$ & Number of initial samples (from label $y$) \\
  $l, m$   & Number of initial samples that fall below the $LB, \theta$\\
  $\alpha, \beta$   & Theoretical fraction of samples that fall below the $\theta, LB$ \\
  $k$   & Additional samples collected under no-exploration case \\
  $k_{e}, k_{d}$   & Additional samples collected under exploration case, where $k_{e}$ and $k_{d}$ represent samples collected in the exploration and disclosed regions \\
  $T$   & Total number of sequential arriving samples \\
  $p_y$   & Label portions in the populations $\mathbb{P}(Y = y)$ \\
  $F^y, F^y_n$   & Theoretical and empirical CDF $\mathbb{P}(X \leq x|Y=y)$ on the full data domain\\
  $G^y, G^y_m$   & Theoretical and empirical CDF $\mathbb{P}(X \leq x|Y=y)$ on the censored region\\
  $E^y, E^y_{m-l+k_1}$   & Theoretical and empirical CDF $\mathbb{P}(X \leq x|Y=y)$ on the exploration region\\
  $K^y, K^y_{n-m+k}$   & Theoretical and empirical CDF $\mathbb{P}(X \leq x|Y=y)$ on the disclosed region\\
  $R(\theta), R_{emp}(\theta)$   & Expected and empirical risk incurred by an algorithm with a decision threshold $\theta$.\\ 
\botrule
\end{tabular}
\end{table}

\section{Error Bounds on Cumulative Distribution Function Estimates (Ex-post Analysis)} \label{sec:split}

Recall that our first goal is to provide an error bound, similar to the DKW inequality, of the discrepancy between the empirical CDF  of feature distribution $F^y_{n_y+k_y}(x)$ and the ground truth CDF $F^y(x)$, for each label $y$. Note that the empirical CDF is found for each label $y$ separately based on its own data samples. Therefore, we drop the label $y$ from our notation throughout this section for simplicity. Further, we first derive the \emph{a posteriori} bounds for given realizations of $k_y$, and develop the \emph{a priori} version of the bound accordingly in Corollary~\ref{cor:a-priori-bound} {(c.f. Remark~\ref{remark:prior-posterior} for the distinction between \emph{a priori} and \emph{a posteriori} cases)}.

We first state the DKW inequality (an extension of the Vapnik–Chervonenkis (VC) inequality for real-valued data) which provides a CDF error bound given IID data. 

\begin{thm}[The DKW inequality \citep{dvoretzky1956asymptotic,massart1990tight}]\label{thm:GC_theorem}
Let $Z_1, \ldots, Z_n$ be IID real-valued random variables with cumulative distribution function $F(z) = \mathbb{P}(Z_1\leq z)$. Let the empirical distribution function be $F_n(z) = \frac{1}{n}\sum_{i=1}^{n}\mathbbm{1}(Z_i\leq z)$. Then, for every $n$ and $\eta {> 0}$,
\[\mathbb{P}\bigg(\sup_{z\in \mathbb{R}} \Big|F(z) - F_n(z)\Big|\geq \eta \bigg) \leq 2\exp{(-2n\eta^2)}~.\]
\end{thm}

\textbf{Failing to address the censored feedback.} The DKW inequality shows how the likelihood that the maximum discrepancy between the empirical and true CDFs exceeds a tolerance level $\eta$ decreases in the number of (IID) samples $n$. In accordance with the strong law of large numbers, the bounds eventually converge to zero as $n$ becomes large. However, when a decision threshold is set—indicating that only samples above the threshold are collected to improve the bounds—the situation changes. In this case, the collected samples are no longer representative of the entire data domain, as they are limited to the disclosed region. Consequently, the bounds associated with the censored region cannot be improved. This results in a persistent gap between the theoretical and empirical CDFs, preventing convergence to zero. 

We now extend the DKW inequality to the case of non-IID data due to censored feedback. We do so by first splitting the data domain into blocks containing IID data, to which the DKW inequality is applicable. Specifically, although the expanded training dataset is non-IID, the decision maker has access to $m$ IID samples in the censored region, and $n-m+k$ IID samples in the disclosed region. Let $G_m$ and $K_{n-m+k}$ denote the corresponding empirical feature distribution CDFs. The DKW inequality can be applied to bound the difference between these empirical CDFs and the corresponding ground truth CDFs $G$ and $K$. 

It remains to identify a connection between the full CDF $F$, and $G$ (the censored CDF) and $K$ (the disclosed CDF), to reach a DKW-type error bound on the full CDF estimate (see Figure~\ref{fig:comparison} for an illustration). This reassembly from the bounds on the IID blocks into the full data domain is however more involved, as it requires us to consider a set of scaling and shifting factors, which are themselves empirically estimated and different from the ground truth values. We will account for these differences when deriving our generalization of the DKW inequality, as detailed in the remainder of this section. All proofs are given in the Appendix~\ref{app:all_proofs}. 

\begin{figure}[ht]
	\centering
    \includegraphics[width=0.45\textwidth]{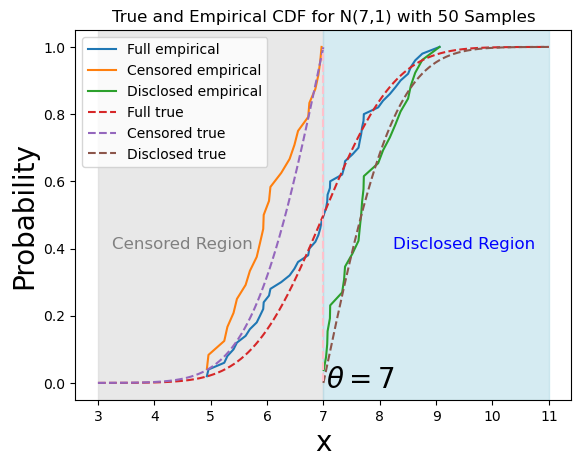}
    \caption{The empirical CDFs $F_{n+k}$ (Full domain), $G_m$ (Censored region), and $K_{n-m+k}$ (Disclosed region), and the theoretical CDFs of $F$, $G,$ and $K$. Experiments based on randomly drawn samples from Gaussian data $N(7,1)$, $\theta=7$, $n=50$, $m=24$, and $k=0$.}
	\label{fig:comparison}
\end{figure}

\subsection{CDF bounds under censored feedback} \label{subsec: no_exploration}

We first present two lemmas that establish how the deviation of $G_m$ and $K_{n-m+k}$ from their corresponding theoretical values, {for a given realization of $m$ data points in the censored region}, relate to the deviation of the full empirical CDF $F_{n+k}$ from its theoretical value $F$. 

\begin{lemma}[Censored Region]\label{lemma:left}
Let $Z = \{x_i|x_i\leq \theta\}$ denote the  $m$ out of $n+k$ samples that are in the censored region. Let $G$ and $G_{m}$ be the theoretical and empirical CDFs of $Z$, respectively. Then, 
{\begin{align*}
    \sup_{x \in (-\infty, \theta)}|F(x)-F_{n+k}(x)| \leq  \sup_{x \in (-\infty, \theta)}\underbrace{\Big|\min\Big(\alpha, \frac{m}{n}\Big)(G(x)  - G_{m}(x))\Big|}_{\text{(scaled) censored subdomain error}} + \underbrace{\Big|\alpha - \frac{m}{n} \Big|}_{\text{scaling error}} .
\end{align*}}
\end{lemma}

The (partial) error bound in this lemma shows the maximum difference between the true $F$ and the empirical $F_{n+k}$ in the censored region (i.e., for $x \in (-\infty, \theta)$) can be bounded by the maximum difference between $G$ and $G_m$, modulated by the \emph{scaling} ($\min(\alpha, \frac{m}{n})$) that is required to map from partial CDFs to full CDFs. 

Specifically, to match the partial and full CDFs, we need to consider the different endpoints of the censored region's CDF and the full CDF at $\theta$, which are $G_m(\theta) = G(\theta) = 1$, $F(\theta) = \alpha$, and $F_{n+k}(\theta) = \frac{m}{n}$, respectively. The first term in the bound above accounts for this by scaling the deviation between the true and empirical partial CDF accordingly. The second term accounts for the error in this scaling since the empirical estimate $\frac{m}{n}$ is generally not equal to the true endpoint $\alpha$. 

The following is a similar result in the disclosed region. 

\begin{lemma}[Disclosed Region]\label{lemma:right}
Let $Z = \{x_i|x_i\leq \theta\}$ denote the  $n-m+k$ out of the $n+k$ samples in the disclosed region. Let $K$ and $K_{n-m+k}$ be the theoretical and empirical CDFs of $Z$, respectively. Then, 
{{\begin{align*} 
    &\sup_{x \in (\theta, \infty)}|F(x)-F_{n+k}(x)| \\
    &\qquad \leq  \sup_{x \in (\theta, \infty)}\underbrace{\Big|\min(1-\alpha,1-\frac{m}{n}) (K(x)  - K_{n-m+k}(x))\Big|}_{\text{(scaled) disclosed subdomain error}} + \underbrace{2\Big|\alpha - \frac{m}{n}\Big|}_{\text{shifting and scaling errors}}
\end{align*}}}
\end{lemma}
 
Similar to Lemma~\ref{lemma:left}, we observe the need for a scaling factor. However, in contrast to Lemma~\ref{lemma:left}, this lemma introduces an additional \emph{shifting error}, resulting in a factor of two in the last term $|\alpha - \frac{m}{n}|$. In particular, we need to consider the different starting points of the disclosed region's CDF and full CDF at $\theta$, which are $K_m(\theta) = K(\theta) = 0$, $F(\theta) = \alpha$, and $F_{n+k}(\theta) =\frac{m}{n}$, respectively, when mapping between the CDFs; one of the $|\alpha - \frac{m}{n}|$ captures the error of shifting the starting point of the partial CDF to match that of the full CDF. 

Building on these lemmas, the following theorem generalizes the well-known DKW inequality to problems with censored feedback under the \emph{ex-post analysis} case. 

\begin{thm}\label{thm:two_subdomains}
Let $x_1, x_2, \ldots, x_n$ be realized initial data samples, drawn IID from a distribution with CDF $F(x)$. Let a (realized) $\theta$ partition the data domain into two regions, such that $\alpha=F(\theta)$, and $m$ of the initial $n$ samples are located below $\theta$. Assume we have collected $k$ additional samples above the threshold $\theta$, and let $F_{n+k}(x)$ denote the empirical CDF estimated from these $n+k$ (non-IID) data. 
Then, for every $\eta, k >0$,
{
\begin{align*}
     &\mathbb{P}\bigg[\sup_{x\in \mathbb{R}} \Big|F(x) - F_{n+k}(x)\Big|  \geq \eta \bigg]  \\
     &\qquad \hspace{1in}\leq \underbrace{2\exp\Big({\tfrac{-2m(\eta-|\alpha - \frac{m}{n}|)^2}{\min\big(\alpha, \frac{m}{n}\big)^2}}\Big)}_{\text{censored region error (constant)}} + \underbrace{2\exp\Big({\tfrac{-2(n-m+k)(\eta-2|\alpha - \frac{m}{n}|)^2}{\min\big(1-\alpha,\frac{n-m}{n}\big)^2}}\Big)}_{\substack{\text{disclosed region error} \\ \text{(decreasing with additional data)}}}
\end{align*}}
\end{thm}

The proof proceeds by applying the DKW inequality to each subdomain, and combining the results using a union bound on the results of Lemmas~\ref{lemma:left} and~\ref{lemma:right}. 

\vspace{0.1in}
\textbf{Numerical illustration of bound behavior for each term in Theorem~\ref{thm:two_subdomains}.} From the above expression, we observe that as $n$ (the number of initial samples across the entire data domain) becomes large, the maximum discrepancy between the theoretical and empirical CDFs decreases, following the strong law of large numbers. In such cases, the effect of censored feedback on the bounds becomes minimal, as the data distribution can already be well estimated. More interestingly, as shown in Fig.~\ref{fig:censored_disclosed_bounds}, when the initial training dataset is small, censored feedback can have a substantial impact on the bounds. To illustrate this, we conduct a numerical experiment using initial random samples $n \in \{50, 200, 500, 1000\}$ drawn from a Gaussian distribution with $\mu = 7$ and $\sigma = 3$. For significance levels $\delta \in \{0.01, 0.05, 0.1\}$, Fig.~\ref{fig:censored_disclosed_bounds} shows that, for any fixed $n$, the CDF bounds for the censored region remain constant since this term does not depend on newly collected samples; thus, improving the bounds requires increasing $n$. Meanwhile, as the number of samples collected under censored feedback increases ($k \rightarrow \infty$), the error term for the disclosed region decreases asymptotically, behaving as $2exp(-2k\eta^2)$. This results in an overall decreasing trend, which is also reflected in the numerical illustration in Fig.~\ref{fig:bound}, where the orange line falls below the blue line. However, unlike the DKW bound, this error bound does not go to zero due to a constant error term from the censored region of the data domain. This means that unless exploration strategies are adopted, arbitrarily good generalization in censored feedback tasks cannot be guaranteed. Finally, we note that the DKW inequality can be recovered as the special case of our Theorem~\ref{thm:two_subdomains} by letting $\theta \rightarrow -\infty$ (which makes $\alpha \approx 0, m \approx 0$). 

\begin{figure}[tbhp!]
	\centering
\includegraphics[width=0.4\textwidth]{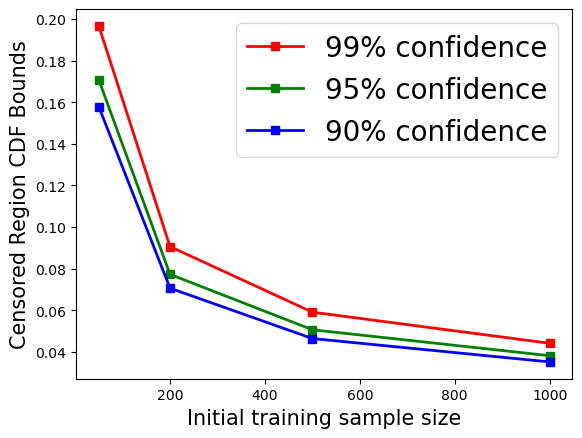}
		\includegraphics[width=0.4\textwidth]{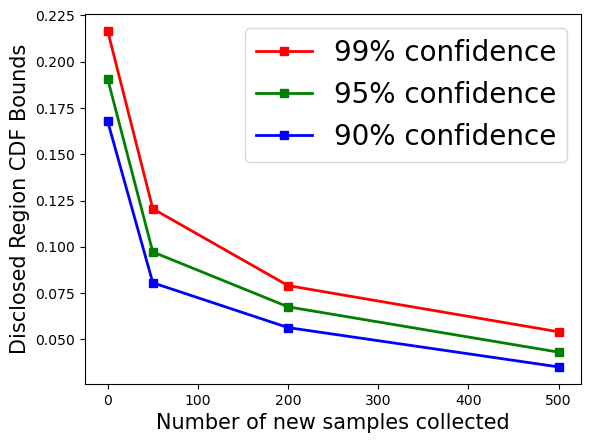}
	\caption{Behavior of the censored and disclosed region error terms.}
	\label{fig:censored_disclosed_bounds}
\end{figure}

\vspace{0.1in}
\textbf{\emph{A priori} bound.} Finally, recall from Remark~\ref{remark:prior-posterior} that instead of waiting to observe that there are exactly $k$ new samples in the disclosed region (i.e., after the random arrival of new data points), a decision maker may want to know the error bound \emph{a priori}, and assess in advance where they will stand after waiting for $T$ data points to arrive (only some of which will fall in the disclosed region). The following corollary provides an error bound under this viewpoint. 

\begin{corollary}\label{cor:a-priori-bound}
Let $x_1, x_2, \ldots, x_n$ be realized initial data samples, drawn IID from a distribution with CDF $F(x)$. Let a (realized) $\theta$ partition the data domain into two regions, such that $\alpha=F(\theta)$, and $m$ of the initial $n$ samples are located below $\theta$. Assume we have waited for $T$ additional samples to arrive, and let $F_{n+T}(x)$ denote the empirical CDF estimated accordingly. 
Then, for every $\eta >0$,
\begin{align*}
     \mathbb{P}\bigg[\sup_{x\in \mathbb{R}} \Big|F(x) - F_{n+T}(x)\Big|  \geq \eta \bigg]  &\leq \underbrace{2\exp\Big({\tfrac{-2m(\eta-|\alpha - \frac{m}{n}|)^2}{\min\big(\alpha, \frac{m}{n}\big)^2}}\Big)}_{\text{censored region error (constant)}} 
     \\
    &\qquad \hspace{-0.2in} +\underbrace{\sum_{k=0}^T 2{T\choose k}(1-\alpha)^k\alpha^{T-k}\exp\Big({\tfrac{-2(n-m+k)(\eta-2|\alpha - \frac{m}{n}|)^2}{\min\big(1-\alpha,\frac{n-m}{n}\big)^2}}\Big)}_{\text{disclosed region error (decreasing with wait time $T$)}}~
\end{align*}
\end{corollary}
The proof is straightforward, and follows from writing the law of total probability for the left-hand side of the inequality by conditioning on the realization $k$ of the samples in the disclosed region. We first note that the {censored region} error term, as expected, is unaffected by the wait time $T$. The second term is the disclosed region error from Theorem~\ref{thm:two_subdomains}; it is decreasing with $T$ as the exponential error terms decrease with $k$, and higher $k$'s are more likely at higher $T$.

\subsection{CDF bounds under censored feedback and exploration} \label{subsec: bounded_exploration}

A commonly proposed method to alleviate censored feedback, as noted in Section~\ref{sec:intro}, is to introduce exploration in the data domain. From the perspective of the CDF error bound, exploration has the advantage of reducing the constant error term in Theorem~\ref{thm:two_subdomains}, by collecting more data samples from the censored region. Formally, we consider (bounded) exploration in the \emph{range} $x\in(LB, \theta)$, where samples in this range are admitted with an exploration \emph{frequency} $\epsilon$. When $LB\rightarrow -\infty$, this is a pure exploration strategy. 

Now, the lowerbound $LB$ and the decision threshold $\theta$ partition the data domain into three IID subdomains (see Figure~\ref{fig:comparison_with_exploration} for an illustration). However, the introduction of the additional \emph{exploration region} $(LB, \theta)$ will enlarge the CDF bounds, as it introduces new scaling and shifting errors when 
reassembling subdomain bounds into full domain bounds. 

\begin{figure}[ht]
	\centering
    \includegraphics[width=0.45\textwidth]{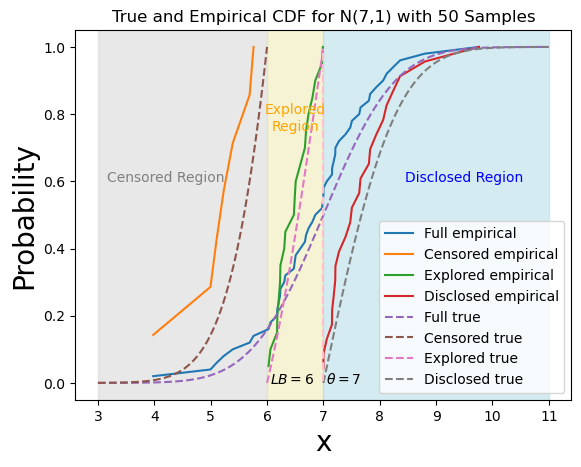}
    \caption{The empirical CDFs $F_{n+k_e+k_d}$ (Full domain), $G_l$ (Censored region), $E_{m-l+k_e}$ (Explored region), and $K_{n-m+k_d}$ (Disclosed region), and the theoretical CDFs of $F, G, E,$ and $K$. Experiments based on randomly drawn samples from Gaussian data $N(7,1)$, $\theta=7$ $LB = 6$, $n=50$, $l=7, m=27$, and $k_e=k_d=0$.}
	\label{fig:comparison_with_exploration}
\end{figure}
Specifically, of the $n$ initial data, let $l$, $m-l$, and $n-m$ of them be in the censored (below $LB$), exploration (between $LB$ and $\theta$), and disclosed (above $\theta$) regions, respectively. Let $\beta=F(LB)$ and $\alpha=F(\theta)$, with initial empirical estimates $\frac{l}{n}$ and $\frac{m}{n}$, respectively. 

As new agents arrive, let $k_e$ and $k_d$ denote the additional samples collected in the exploration range and disclosed range, respectively. One main difference of this setting with that of Section~\ref{subsec: no_exploration} is that as additional samples are collected, the empirical estimate of $\alpha$ can be re-estimated. Accordingly, we present a lemma similar to Lemmas~\ref{lemma:left} and~\ref{lemma:right} for the exploration region. 

\begin{lemma}[Exploration Region]\label{lemma:exploration}
Let {$Z = \{x_i | \text{LB} \leq x_i \leq \theta\}$} denote the $m-l+k_e$ samples out of the $n+k_e+k_d$ samples that are in the exploration range. Let $E$ and $E_{m-l+k_e}$ be the theoretical and empirical CDFs of $Z$, respectively. Then,  
{\begin{align*}
    &\sup_{x \in (\text{LB}, \theta)}|F(x)-F_{n+k_e+k_d}(x)| \leq \underbrace{\Big|\beta - \frac{l}{n} \Big|}_{\text{shifting error}} + \underbrace{\Big|\alpha-\beta - \frac{n-l}{n}\frac{m-l+k_e}{n-l+k_e+\epsilon k_d} \Big|}_{\text{re-estimated scaling error}} \\
    &\qquad \hspace{1in}
    + \underbrace{\sup_{x \in (\text{LB}, \theta)}\Big|\min\big(\alpha-\beta, \tfrac{n-l}{n}\tfrac{m-l+k_e}{n-l+k_e+\epsilon k_d}\big)(E(x) - E_{m-l+k_e}(x))\Big|}_{\text{scaled exploration subdomain error}}
\end{align*}}
\end{lemma}

Observe that here, we need both scaling and shifting factors to relate the partial and full CDF bounds, as in Lemma~\ref{lemma:right}, but with an evolving scaling error as more data is collected. In particular, the initial empirical estimate $\frac{m}{n}$ is updated to $ \frac{l}{n} +  \frac{n-l}{n}\frac{m-l+k_e}{n-l+k_e+\epsilon k_d}$ after the observation of the additional $k_e$ and $k_d$ samples. 

We now extend the DKW inequality to the setting where data is collected under censored feedback with exploration. Similar to the no-exploration case, the following theorem can be considered under the \emph{ex-post analysis} case and the \emph{ex-ante analysis} case. Here, we present the results for the \emph{ex-post analysis} case and defer the other case to Appendix~\ref{app:partial_full_exploration_bounds}.

\begin{thm}\label{thm:three_subdomains}
Let $x_1, x_2, \ldots, x_n$ be realized initial data samples, drawn IID from a distribution with CDF $F(x)$. Let (realized) $LB$ and $\theta$ partition the domain into three regions, such that $\beta=F(LB)$ and $\alpha=F(\theta)$, with $l$ and $m$ of the initial $n$ samples located below $LB$ and $\theta$, respectively. Assume we have collected an additional $k_e$ samples between $LB$ and $\theta$, under an exploration probability $\epsilon$, and an additional number of $k_d$ samples above $\theta$. Let $F_{n+k_e+k_d}(x)$ denote the empirical CDF estimated from these $n+k_e+k_d$ non-IID samples. Then, for every {$\eta, k_e, k_d >0$},
{\small
\begin{align*}
    \mathbb{P}\bigg[\sup_{x\in \mathbb{R}} \Big|F(x) - F_{n+k_e+k_d}(x)\Big|\geq \eta \bigg] &\leq 
    \underbrace{2\exp\Big({\tfrac{-2l(\eta-|\beta - \frac{l}{n} |)^2}{\min\big(\beta, \frac{l}{n}\big)^2}}\Big)}_{\text{(still) censored region error (constant)}} \\
    &\qquad \hspace{-0.3in} + \underbrace{2\exp\Big({\tfrac{-2(m-l+k_e)\big(\eta-|\beta - \frac{l}{n}| - \big|\alpha -\beta - \frac{n-l}{n}\frac{m-l+k_e}{n-l+k_e+\epsilon k_d}\big|\big)^2}{\min\big(\alpha-\beta,\frac{n-l}{n}\frac{m-l+k_e}{n-l+k_e+\epsilon k_d}\big)^2}}\Big)}_{\text{exploration region error (decrease with $k_e$)}} \\
    &\qquad \hspace{-0.3in} + \underbrace{2\exp\Big({\tfrac{-2(n-m+k_d)\big(\eta-2\big|\alpha - \frac{l}{n} -  \frac{n-l}{n}\frac{m-l+k_e}{n-l+k_e+\epsilon k_d}\big|\big)^2}{\min\big(1-\alpha,\frac{n-l}{n}\frac{n-m+\epsilon k_d}{n-l+k_e+\epsilon k_d}\big)^2}}\Big)}_{\text{disclosed region error (decrease with $k_d$)}}.
\end{align*}
}
\end{thm}

Comparing this expression with Theorem~\ref{thm:two_subdomains}, we first note that the last term corresponding to the error bound in the disclosed region are similar when setting $k=k_d$, with the difference being in the impact of re-estimating $\alpha$. 

Theorem~\ref{thm:three_subdomains} provides an extension of the DKW inequality to account for censored feedback and exploration, introducing three distinct regions: the (still) censored region $(-\infty, LB)$, the exploration region $(LB, \theta)$, and the disclosed region $(\theta, \infty)$. The (still) censored region contributes a constant error term dependent on $l$, the number of initial samples in this region, due to the absence of exploration. The exploration region introduces $k_e$, the number of samples collected under an exploration probability $\epsilon$, which reduces the error in this region as $k_e \rightarrow \infty$, ultimately approaching zero. In contrast, the disclosed region contributes an error based on $n-m$, the number of initial samples above $\theta$, and $k_d$, the number of additional samples collected in this region. As $k_d$ increases, the error in the disclosed region also diminishes. 

A key insight from Theorem~\ref{thm:three_subdomains} is that although there can still be a non-vanishing error term in the (still) censored region, additional samples collected in the exploration and disclosed regions can reduce their respective error terms. Similar to Theorem~\ref{thm:two_subdomains}, due to the newly collected samples $k_e$ and $k_d$, the error term for the exploration and disclosed region decreases asymptotically, behaving as $2exp(-2k_e\eta^2)$ and $2exp(-2k_d\eta^2)$, respectively, similar to the findings in Fig.~\ref{fig:censored_disclosed_bounds}. Further, if we adopt pure exploration ($LB\rightarrow -\infty$, which makes $\beta \approx 0, l \approx 0$), the first term will vanish as well (however, note that pure exploration may not be a feasible option if exploration is highly costly). Lastly, we note that an \emph{a priori} version of this bound can be derived using similar techniques to that of Corollary~\ref{cor:a-priori-bound}.

\subsection{When will exploration improve generalization guarantees?} \label{subsec: illustration}

It might seem at first sight that the new vanishing error term in the exploration range of Theorem~\ref{thm:three_subdomains} necessarily translates into a tighter error bound than that of Theorem~\ref{thm:two_subdomains} when exploration is introduced. Nonetheless, the shifting and scaling factors, as well as the introduction of an additional union bound, enlarge the CDF error bound. In this section, we elaborate on the trade-off between these factors, and evaluate when the benefits of exploration outweigh its drawbacks in providing error bounds on the data CDF estimates. 

We begin by presenting two propositions that assess the change in the bounds of Theorems~\ref{thm:two_subdomains} and~\ref{thm:three_subdomains} as a function of the severity of censored feedback (as measured by $\theta$) and the exploration frequency $\epsilon$. 

\begin{proposition}~\label{prop:B(LB) < B(theta)}
Let $B(\theta)$ denote the error bound in Theorem~\ref{thm:two_subdomains}, and assume the conditions of that theorem hold. Assume also that we can collect an additional $k=O(n)$ samples above the threshold. Then, $B(\theta)$ is increasing in $\theta$. 
\end{proposition}

\begin{proposition}~\label{prop:B(LB, theta, 1) < B(LB, theta, epsilon)}
Let $B^e(LB, \theta, \epsilon)$ denote the error bound in Theorem~\ref{thm:three_subdomains}, and assume the conditions of that theorem hold. Then, $B^e(LB, \theta, \epsilon)$ is decreasing in $\epsilon$.
\end{proposition}

In words, as intuitively expected, these propositions state that the generalization bounds worsen (i.e., are less tight) when the censored feedback region is larger, and that they can be improved (i.e., made more tight) as the frequency of exploration increases.  
\begin{figure}[ht]
	\centering
    \includegraphics[width=0.45\textwidth]{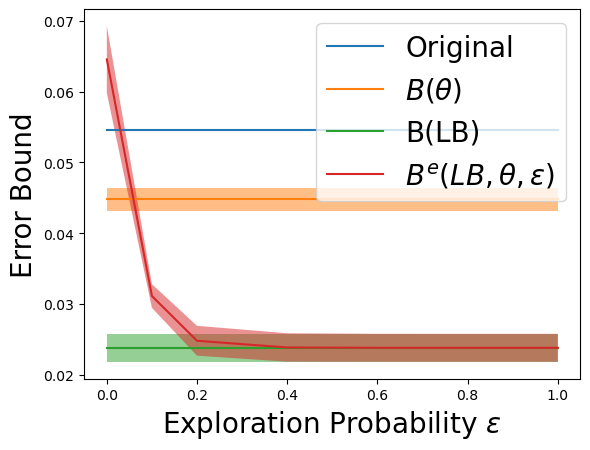}
    \caption{A minimum exploration frequency is needed to tighten the CDF error bound.}
	\label{fig:bound}
\end{figure}

\textbf{Numerical illustration}. We also conduct a numerical experiment to illustrate the bounds derived in Theorems~\ref{thm:two_subdomains} and~\ref{thm:three_subdomains}. We proceed as follows: $n=8000$ random samples are drawn from a Gaussian distribution with $\mu=7$ and $\sigma=3$, with an additional {$T=40000$} samples arriving subsequently, randomly sampled from across the entire data domain. We set the significance level $\delta=0.01$, the threshold $\theta=8$, and the lower bound $LB=6$. We run the experiment 5 times and report the error bounds from Theorems~\ref{thm:two_subdomains} and \ref{thm:three_subdomains} accordingly.\footnote{While we evaluate the \emph{a posteriori} bounds under the realizations of the new samples falling in the disclosed and exploration regions ($k$, $k_1$, and $k_2$), we show the averages of the bounds over multiple runs, which can be viewed as an approximation of the \emph{a priori} version of the bounds. See also Appendix~\ref{app:more-experiments} for a comparison of the bounds from Corollary~\ref{cor:a-priori-bound} and Theorem~\ref{thm:two_subdomains}.}

In Figure~\ref{fig:bound}, the ``original'' (blue) line represents the DKW CDF bound of the initial samples without additional data. The ``$B(\theta)$'' (orange) line and ``$B(LB)$'' (green) line represent the CDF bound in Theorem~\ref{thm:two_subdomains} without exploration, where the decision threshold is at $\theta$ and $LB$, respectively. The ``$B^e(LB, \theta,\epsilon)$'' (red) line represents the bound in Theorem~\ref{thm:three_subdomains} with exploration probability $\epsilon$.

From Figure \ref{fig:bound}, we first observe that the green line ($B(LB)$, which observes new samples with $x\geq LB=6$) provides a tighter bound than the orange line ($B(\theta)$, which observes new samples with $x\geq \theta=8$), with both providing tighter bounds than the blue line (original DKW bound, before any new samples are observed). This is shown by that the green line is below the orange line, which is also below the blue line. This improvement is due to collecting more samples from the disclosed region results in a decrease in the CDF error bound, as noted by Proposition~\ref{prop:B(LB) < B(theta)}. Additionally, we can observe from the trajectory of the red line ($B^e(LB, \theta,\epsilon)$, which observes a fraction $\epsilon$ of new samples from $(LB, \theta)$, and all new samples above $\theta$) that introducing exploration enlarges the CDF error bound due to the additional union bound, but it also enables the collection of more samples, leading to a decrease in the CDF error bound as $\epsilon$ increases (evidenced by the red line is decreasing when $\epsilon$ increases along the x-axis); note that this observation aligns with Proposition~\ref{prop:B(LB, theta, 1) < B(LB, theta, epsilon)}. 
        
Notably, we see that a minimum level of exploration probability $\epsilon$ (accepting around 10\% of the samples in the exploration range) is needed to improve the CDF bounds over no exploration. Note that this may or may not be feasible for a decision maker depending on the costs of exploration (see also Section~\ref{sec:how-to-explore}). However, if exploration is feasible, we also see that accepting around 20\% of the samples in the exploration range can be sufficient to provide bounds nearly as tight as observing all samples in the exploration range. In other words, we can see from Fig.~\ref{fig:bound} that the $\epsilon$ is around 10\% when the red line crosses the orange line, and it is around 20\% when the red line is close to the green line.

\subsection{How to choose an exploration strategy?}\label{sec:how-to-explore} 

We close this section by discussing potential considerations in the choice of an exploration strategy in light of our findings. Specifically, a decision maker can account for a tradeoff between \emph{the costs of exploration} and \emph{the improvement in the error bound} when choosing its exploration strategy. Recall that the exploration strategy consists of selecting an exploration lowerbound/range $LB$ and an exploration probability $\epsilon$. 
Formally, the decision maker can solve the following optimization problem to choose these parameters:
\begin{align}
    \max_{\epsilon\in[0,1], LB\in[0,\theta]} ~~ \big(B(\theta)-B^e(LB, \theta, \epsilon)\big) -  C(LB, \theta, \epsilon)~,
    \label{eq:optimizing-exploration}
\end{align}
where $B(\theta)$ and $B^e(LB, \theta, \epsilon)$ denote the error bounds in Theorems~\ref{thm:two_subdomains} and~\ref{thm:three_subdomains}, respectively,  
and $C(LB, \theta, \epsilon)$ is an exploration cost which is non-increasing in $(\theta-LB)$ (reducing the exploration range will weakly decrease the costs) and non-decreasing in $\epsilon$ (exploring more samples will weakly increase the cost). As an example, the cost function $C(LB, \theta, \epsilon)$ can be given by
\begin{align}
C(LB, \theta, \epsilon) =  \epsilon \int_{LB}^{\theta} e^{\frac{\theta-x}{c}} f^0(x) \mathrm{d}x.
\label{eq:cost-function-example}
\end{align}
In words, unqualified (costly) samples at $x$ have a density $f^0(x)$, and when selected (as captured by the $\epsilon$ multiplier), they incur a cost $e^{\frac{\theta-x}{c}}$, where $c>0$ is a constant. Notably, observe that the cost is increasing as the sample $x$ gets further away from the threshold $\theta$. For instance, in the bank loan example, this could capture the assumption that individuals with lower credit scores default on a larger portion of their loans. 

As noted in Proposition~\ref{prop:B(LB, theta, 1) < B(LB, theta, epsilon)}, $B^e(LB, \theta, \epsilon)$ is decreasing in $\epsilon$; coupled with any cost function $C(LB, \theta, \epsilon)$ that is (weakly) increasing in $\epsilon$, this means that the decision maker's objective function in \eqref{eq:optimizing-exploration} captures a tradeoff between reducing generalization errors and modulating exploration costs. 

The optimization problem in \eqref{eq:optimizing-exploration} can be solved (numerically) by plugging in for the error bounds from Theorems~\ref{thm:two_subdomains} and~\ref{thm:three_subdomains} and an appropriate cost function (e.g., \eqref{eq:cost-function-example}). For instance, in the case of the numerical example of Fig.~\ref{fig:bound}, under the cost function of \eqref{eq:cost-function-example} with $c=5$, and fixing $LB=6$, the decision maker should select $\epsilon=11.75\%$. 

Another potential solution for modulating exploration costs is to use multiple exploration subdomains, each characterized by an exploration range $[LB_i, LB_{i-1})$, and with a higher exploration probability $\epsilon_i$ assigned to the subdomains closer to the decision boundary (which are less likely to contain high cost samples). For instance, with the choice of $b$ subdomains, the cost function of \eqref{eq:cost-function-example} would change to (the lower) cost: 
\begin{align}
C(\{LB_i\}_{i=1}^b, \theta, \{\epsilon_i\}_{i=1}^b) =  \sum_{i=1}^{b} \epsilon_i \int_{LB_i}^{LB_{i-1}} e^{\frac{\theta-x}{c}} f^0(x) \mathrm{d}x.
\label{eq:cost-function-example-modified}
\end{align}
It is worth noting that while this approach can reduce the costs of exploration, it will also weaken generalization guarantees when we reassemble the $b$ exploration subdomains' bounds back into an error bound of the full domain (similar to what was observed in Fig.~\ref{fig:bound} for $b=1$). This again highlights a tradeoff between improving learning error bounds and restricting the costs of data collection.

\section{Generalization Error Bounds under Censored Feedback (Ex-post Analysis)} \label{sec:theoretical}

In this section, we use the CDF error bounds from Section~\ref{sec:split} to characterize the generalization error of a classification model that has been learned from data collected under censored feedback. Specifically, we will first establish a connection between the generalization error of a classifier (the quality of its learning) and the CDF error bounds on its training dataset (the quality of its data). With this relation in hand, we can then use any of the CDF error bounds from Theorems~\ref{thm:GC_theorem}-\ref{thm:three_subdomains} to bound how well algorithms learned on data suffering from censored feedback (and without or with exploration) can generalize to future unseen data.

Formally, we consider a 0-1 learning loss function $\mathcal{L}: \mathcal{Y} \times \mathcal{Y} \rightarrow \{0,1\}$. Denote $R(\theta)=\mathbb{E}_{XY}\mathcal{L}(f_{\theta}(X), Y)$ as the expected risk incurred by an algorithm with a decision threshold $\theta$. Similarly, we define the empirical risk as $R_{emp}(\theta)$. The \emph{generalization error bound} is an upper bound to the error $|R(\hat{\theta}) - R_{emp}(\hat{\theta})|$, where $\hat{\theta}$ is the minimizer of the empirical loss, i.e., $\hat{\theta}:= \arg \min_{\theta} R_{emp}(\theta)$. In words, the bound provides a (statistical) guarantee on the performance $R(\hat{\theta})$, when using the learned $\hat{\theta}$ on unseen data, relative to the performance $R_{emp}(\hat{\theta})$ assessed on the training data. Our objective is to characterize this bound under censored feedback, and to evaluate how utilizing (pure or bounded) exploration can improve the bound. 

Recall that the decision maker starts with a training data containing $n_y$ IID samples from each label $y$, drawn from an underlying distribution with CDF $F^y(x)$. Let $n=n_0+n_1$ denote the size of the initial training data. Then, the expected loss of a binary classifier with decision threshold $\theta$ is given by,
\[R(\theta)= \mathbb{E}_{XY}\mathcal{L}(f(X),Y) = p_1F^1(\theta) + p_0(1-F^0(\theta))~,\]
while the empirical loss $R_{emp}(\theta)$ is given by,
\begin{align*}
    R_{emp}(\theta) =\frac{n_1}{n}\frac{1}{n_1}\sum_{(x_i,y_i)} \mathbbm{1}{\{x_i \leq \theta,y_i = 1\}} + \frac{n_0}{n}\Big(1-\frac{1}{n_0}\sum_{(x_i,y_i)} \mathbbm{1}{\{x_i \leq \theta,y_i = 0\}}\Big).
\end{align*}
Similarly, if the decision maker can collect an additional $k_y$ samples of agents with features above the threshold $\theta$, the above empirical risk expression can be updated accordingly, by considering the $n_y+k_y$ samples available from each label $y$.

We detail the derivations of these expressions in Appendix~\ref{app:loss_derivation}. Using these expressions of the expected and empirical risks, the following theorem provides an upper bound on the generalization error $|R(\hat{\theta}) - R_{emp}(\hat{\theta})|$ as a function of the CDF error bound, where $\hat{\theta}$ denotes the minimizer of the empirical loss, i.e., $\hat{\theta}:=\arg\min_{\theta} R_{emp}(\theta)$. 

\begin{thm}~\label{thm:error}
Consider a threshold-based classifier $f_{\hat{\theta}}(x):\mathcal{X}\rightarrow \{0,1\}$, determined from a dataset containing $n_y$ initial IID training samples from each label $y$, with $n=n_0+n_1$, under a 0-1 loss function. Let $p_y$ denote the proportion of samples from label $y$. Subsequently, due to the censored feedback, the algorithm collects $k_y$ additional samples from each label $y$. Let $F^y$ and $F_m^y$ denote the CDFs and empirical CDFs, respectively, given $m$ samples from label $y$ agents. 
Then, with probability at least $1-2\delta$,
\begin{align*}
    &\Big|R(\hat{\theta}) - R_{emp}(\hat{\theta})\Big| \leq 3\Big|p_0 - \frac{n_0}{n}\Big| + \sum_{y\in \{0,1\}} \min\Big(p_y, \frac{n_y}{n}\Big)\sup_{\theta}\Big|F^y(\theta) - F^y_{n_y+k_y}(\theta)\Big|~.
\end{align*}
\end{thm}
First, we note that tightening the CDF error bounds leads to tightening the generalization error guarantees. More specifically, using this theorem together with Theorems~\ref{thm:GC_theorem}, \ref{thm:two_subdomains}, and \ref{thm:three_subdomains}, we can provide a generalization error guarantee for an algorithm in terms of the number of available data samples in its training data from each label and in different parts of the data domain, particularly when future data availability is non-IID due to censored feedback. 

For instance, the DKW inequality can be alternatively expressed as follows: given $n_y$ IID samples from a label $y$, with probability at least $1-\delta$, the following inequality holds:
\[\sup_{z}\Big|F(z) - F^y_{n_y}(z)\Big| \leq {\sqrt{\frac{\log \frac{2}{\delta}}{2n_y}}~.}\]
Using this expression in Theorem~\ref{thm:error}, we conclude that (without censored feedback, or with pure exploration with $\epsilon=1$) with probability at least $1-2\delta$, 
\begin{align*}
    \Big|R(\hat{\theta}) - R_{emp}(\hat{\theta})\Big|   &\leq 3\Big|p_0 - \frac{n_0}{n}\Big| + \sum_{y\in \{0,1\}} \min\Big(p_y, \frac{n_y}{n}\Big){\sqrt{\frac{\log \frac{2}{\delta}}{2n_y}}}.
\end{align*}

We can similarly specialize Theorem~\ref{thm:error} to tasks with censored feedback by linking it with Theorems~\ref{thm:two_subdomains} and ~\ref{thm:three_subdomains}. Given the complexity of the CDF error bounds under censored feedback, while we cannot derive a closed-form expression for the bound as done for the DKW inequality, we can compute the bounds numerically, as shown in the next section.

\section{Numerical Experiments} \label{sec:numerical}

\subsection{CDF error bounds}
We first illustrate our derived bounds (with $\delta = 0.015$) on the empirical CDF. We start with 50 random samples from a Gaussian distribution N(7,1). Next, 200 new samples are drawn from the same distribution, with all samples with features $x \geq \theta=7$ accepted, and samples with features $LB=6 \leq x \leq \theta$ accepted with a probability $\epsilon\in\{0, 0.5, 1\}$; higher values of $\epsilon$ represent less censored feedback ($\epsilon=1$ means no censored feedback). 

\begin{figure}[ht]
	\centering
    \includegraphics[width=0.6\textwidth]{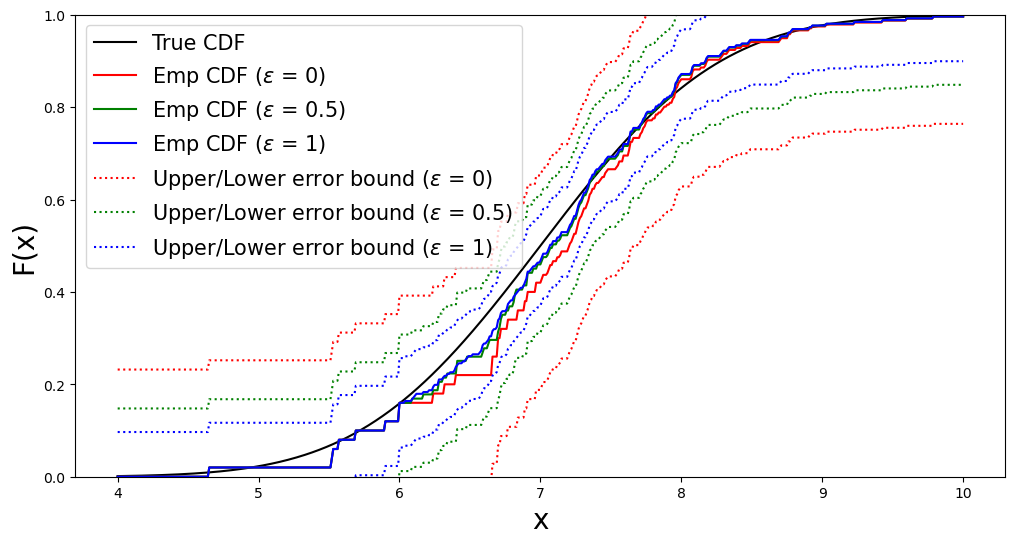}
    \caption{CDF error bounds evaluated under different exploration probabilities.}
	\label{fig:DKW}
\end{figure}

From Figure~\ref{fig:DKW}, we first note that our bounds (the dotted lines) effectively enclose the true distribution, evidenced by that the true distribution is upper and lower bounded by the dotted lines. We also note the distinction between empirical CDFs in the disclosed region ($x\geq7$) and the censored region ($x \leq 7$): as intuitively expected, empirical CDFs (solid lines) in the disclosed region are ``smoother'' compared to those in the censored region. Furthermore, as $\epsilon$ (exploration) increases, we overcome censored feedback in the exploration region, resulting in more accurate empirical estimates. Additionally, as $\epsilon$ increases, our error bounds improve (i.e., more tightly enclose the true CDF). 

\subsection{Model generalization error bounds: real-world data}

We now illustrate the ability of our generalization error bounds (derived in Theorem~\ref{thm:error}) in providing guarantees on the error of the learned models from data affected by censored feedback, using experiments on a real-world dataset: \emph{FICO} \citep{hardt2016equality}, \emph{Retiring Adult} \citep{ding2021retiring}, \emph{Adult} \citep{Dua:2019}.

\textbf{Experiments on \emph{FICO} dataset.} The \emph{FICO}
dataset is used to predict whether an individual will default. It includes one-dimensional features (e.g., credit scores) with a specific focus on the distribution information of the credit scores. We employ a logistic regression algorithm and 0-1 loss for the classification task, and compare the generalization error across different exploration probabilities ($\epsilon=\{0.5, 1\}$). We start with a 1000 training data samples. A total of 175000 new samples arrive throughout the experiment; in addition to accepting all samples with feature $x \geq \hat{\theta}$, the algorithm also accepts some samples that fall below $\hat{\theta}$. We report our experiment results for an average of 5 runs, where the randomness comes from the order of samples arrived and the exploration. 

\begin{figure}[ht]
	\centering
    \includegraphics[width=0.45\textwidth]{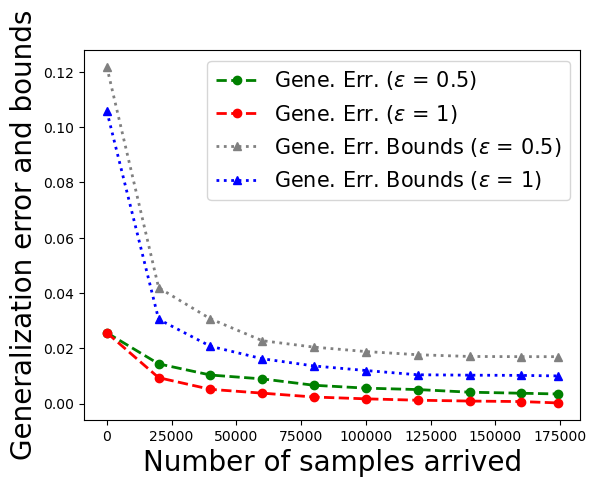}
    \caption{Generalization error and bounds using \emph{FICO} dataset.}
	\label{fig:GEB_FICO}
\end{figure}

\textbf{Experiments on \emph{Retiring Adult} dataset.} The \emph{Retiring Adult} census dataset is used to predict whether an individual can earn more than \$50k/year, based on a multi-dimensional feature set. Similar to the experiments on \emph{FICO} dataset, we employ a logistic regression algorithm and 0-1 loss for the classification task, and compare the generalization error across different exploration probabilities ($\epsilon=\{0.5, 1\}$). A total of 1600000 new samples arrive throughout the experiment. We report our experiment results for an average of 5 runs, where the randomness comes from the order of samples arrived and the exploration. 

\begin{figure}[ht]
	\centering
    \includegraphics[width=0.45\textwidth]{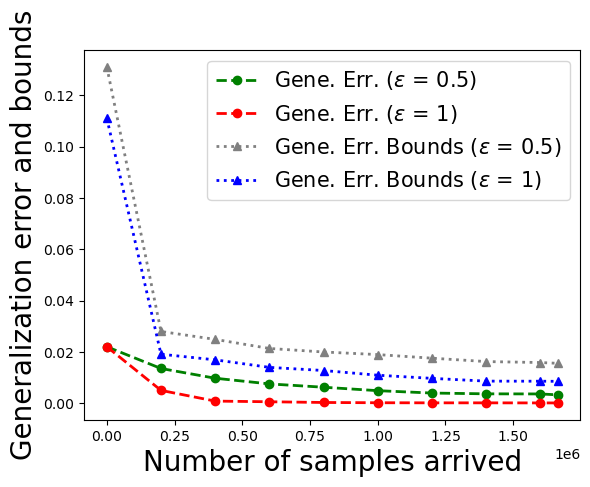}
    \caption{Generalization error and bounds using \emph{Retiring Adult} dataset.}
	\label{fig:GEB_Retiring}
\end{figure}

In Fig.~\ref{fig:GEB_FICO} and \ref{fig:GEB_Retiring}, the y-axis represents the generalization error and its bounds, where a lower value is preferable. A smaller value of the bounds indicates tighter bounds enclosing the generalization error curve. From the experiment results from both \emph{FICO} and \emph{Retiring Adult} datasets in Fig.~\ref{fig:GEB_FICO} and \ref{fig:GEB_Retiring}, we can see that our bounds (shown in gray and blue) can effectively contain the true generalization errors of the model (for both $\epsilon = \{0.5, 1\}$). Furthermore, we can see that when the exploration probability $\epsilon$ is increased, the bounds get tighter (the blue line is below the gray line) due to the additional samples explored during data collection. 

\textbf{Experiments on \emph{Adult} dataset.} The \emph{Adult} census dataset is similar to the \emph{Retiring Adult} dataset, but it has smaller amount of samples. It is also used to predict whether an individual can earn more than \$50k/year, based on a multi-dimensional feature set. A total of 45000 new samples arrive throughout the experiment; We report our experiment results for an average of 5 runs, where the randomness comes from the order of samples arrived and the exploration. In addition, we further consider the model is updated as new samples are collected. Therefore, in the following experiments using \emph{Adult} dataset, we also assess the performance of our bounds based on whether we adaptively update the decision threshold $\hat{\theta}$ with new samples.

\begin{figure}[ht]
	\centering
    \includegraphics[width=0.4\textwidth]{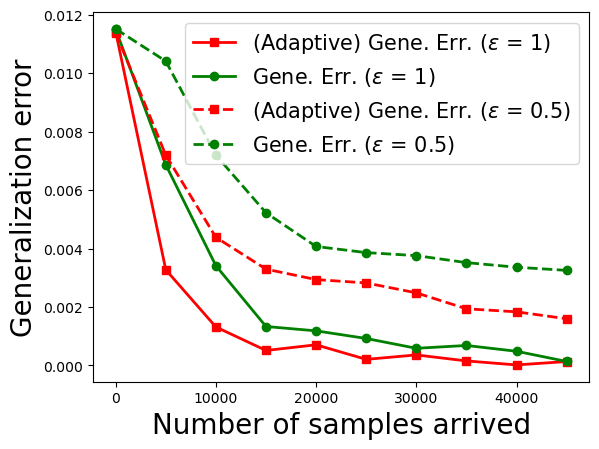}
    \includegraphics[width=0.4\textwidth]{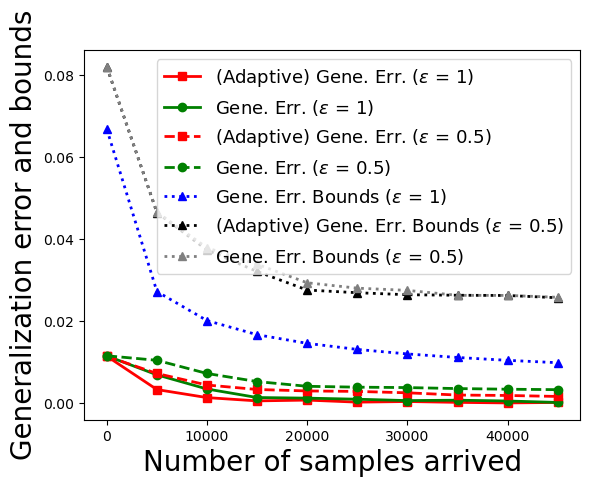}
    \caption{Generalization error with(out) an adaptively updated model ($\hat{\theta}$) and varying exploration ($\epsilon$).}
	\label{fig:GEB_adaptive}
\end{figure}

From Figure~\ref{fig:GEB_adaptive}, we observe that as the decision threshold $\hat{\theta}$ is adaptively updated when more samples are collected, it has even better generalization performance compared to a non-adaptive decision threshold (evidenced by the red curve being lower than the green curve) This is expected as a refined decision threshold yields better performance on unseen data. Further, for the generalization error bounds (dotted lines in the right panel), we see that our bounds effectively contain the true generalization errors of the model for both the fixed model and adaptively updated model cases (all dotted lines are above the red/green curves). Notably, in the presence of censored feedback, we observe that the generalization error bound with adaptive updating is tighter than the non-adaptive one (the black curve is below the gray curve), pointing to a potential future research direction for further improving our bounds. 

\subsection{Comparison with existing generalization error bounds} 

\begin{figure}[ht]
	\centering
    \includegraphics[width=0.45\textwidth]{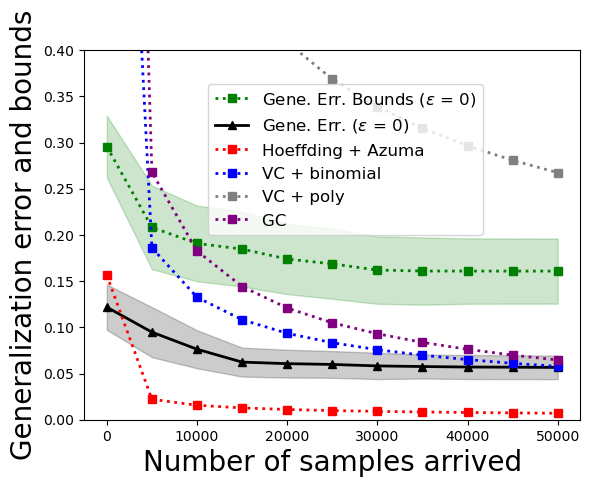}
    \caption{Existing bounds fail to capture generalization error when there is censored feedback.}
	\label{fig:benchmark}
\end{figure}

We now compare the performance of our bounds with a number of existing generalization error bounds, and show that by failing to account for censored feedback, prior works fail to correctly capture how well a model learned on data suffering from censored feedback generalizes to unseen data. We consider the following four benchmarks: The `Hoeffding + Azuma' bounds represent those derived from Hoeffding and Azuma inequalities \citep{hoeffding1994probability,azuma1967weighted}. The `VC + binomial' bounds are VC generalization bounds \citep[Thm 2.5]{vapnik2015uniform,Mostafa2012} where the shatter coefficient is bounded through the binomial theorem. The `VC + poly' bounds represent VC generalization bounds \citep[Thm 13.11]{vapnik2015uniform,devroye2013probabilistic} applicable to any linear classifier whose empirical error is minimal, where the shatter coefficient is bounded by a polynomial function. Lastly, the `GC' bounds \citep{glivenko1933sulla, cantelli1933sulla} are derived based on the Glivenko-Cantelli Theorem for a threshold classifier and 0-1 loss.

We conduct this experiment on synthetic data. We start with 50 initial training samples for each label $y \in \{0, 1\}$ randomly drawn from Gaussian distributions N(9,1) and N(10,1), respectively. The decision threshold $\hat{\theta}$ is selected to be the one minimizing the misclassification error on the training data. Then, a total of 50000 new samples arrive throughout the experiment. They will be accepted if the feature $x \geq \hat{\theta}$, otherwise, they are rejected. We run the experiments 5 times and report the average results with corresponding error bars. From Figure~\ref{fig:benchmark}, we can clearly see that the `Hoeffding-Azuma' (red), `VC+binomial' (blue), and `GC' (purple) bounds are inadequate for accurately estimating the true generalization error guarantees of the model. This inadequacy is demonstrated by the fact that all three bounds cross the true error (black) line as new samples are collected under the presence of censored feedback. For the `VC+poly' (gray) bound, although it provides a very loose estimate compared to our bounds for the given number of new samples—evidenced by the gray bounds being above our green bounds—it ultimately exhibits similar behavior to the other three benchmarks, in that it will go lower than the true generalization error. 

\section{Error Bounds for the Ex-ante Analysis Case}\label{sec:full_random}

In the previous sections, we explored the \emph{ex-post analysis} case where the initial training data is collected in advance, leading to a realized (known) decision threshold. We now extend our study to cases where the initial training data may be any collection of (uniformly drawn) random samples from the underlying data distributions. Consequently, the decision threshold becomes a data-dependent random variable. 

First, assume the distribution of this random variable $\theta$, denoted $\mathbb{P}(\theta)$, is known. The following is an extension of the DKW inequality to problems with censored feedback, accounting for randomness in the initial training data with a data-dependent $\theta$. (That is, an extension of Theorem~\ref{thm:two_subdomains} to the ex-ante analysis case.) 
 
\begin{corollary}\label{thm:two_subdomains_full}
Let $x_1, x_2, \ldots, x_n$ be random initial data samples, drawn IID from a distribution with CDF $F(x)$. Let the optimal $\theta$, derived from a collection of initial data samples with \emph{both} labels, partition the data domain into two regions, such that $\alpha=F(\theta)$, and $m$ of the initial $n$ samples are located below $\theta$. Assume we have collected $k$ additional samples above the threshold $\theta$, and let $F_{n+k}(x)$ denote the empirical CDF estimated from these $n+k$ (non-IID) data. 
Then, for every $\eta, k >0$,
\begin{align*}
    &\hspace{-0.15in}\mathbb{P}\bigg[\sup_{x\in \mathbb{R}} \Big|F(x) - F_{n+k}(x)\Big|  \geq \eta \bigg] \\
     &\qquad \hspace{-0.35in} \leq \int_{-\infty}^{\infty}  \Bigg (\sum_{m=1}^{n-1}2{n\choose m}\alpha^m(1-\alpha)^{n-m}\Big(\exp\Big({\tfrac{-2m(\eta-|\alpha - \frac{m}{n}|)^2}{\min\big(\alpha, \frac{m}{n}\big)^2}}\Big) \\
    &\qquad + \exp\Big({\tfrac{-2(n-m+k)(\eta-2|\alpha - \frac{m}{n}|)^2}{\min\big(1-\alpha,\frac{n-m}{n}\big)^2}}\Big)\Big) + 2(1-\alpha)^n\Big(1+\exp\Big(\tfrac{-2(n+k)(\eta-2\alpha)^2}{\big(1-\alpha\big)^2}\Big)\Big)\\
    &\qquad +2\alpha^n\Big(\exp\Big(\tfrac{-2n(\eta-|\alpha-1|)^2}{\alpha^2}\Big)+1\Big)\Bigg ) \mathbb{P}(\theta)d\theta.
\end{align*}
\end{corollary}

The proof of Corollary~\ref{thm:two_subdomains_full} is straightforward, and follow from the law of total probability for the bounds by conditioning on the realization $m$ of the samples in the censored region. In corner cases where no samples fall within the censored or disclosed regions, the bounds are determined by applying the DKW inequality directly in the absence of samples. Lastly, we note that an \emph{a priori} version of both bounds can be derived using similar techniques to that of Corollary~\ref{cor:a-priori-bound}.

We now return to the question of knowledge of the probability density $\mathbb{P}(\theta)$, which is needed to evaluate these bounds. 

\paragraph{The probability density $\mathbb{P}(\theta)$} 
\textbf{The challenge.} As the numerical example in Fig.~\ref{fig:p_theta} illustrates, the distribution of the threshold $\theta$ will depend on both the underlying data distributions, as well as the number of samples $(n_0, n_1)$ drawn from each feature-label distribution. In particular, in this numerical example, we consider different combinations of $(n_0, n_1) \in \{(100, 50), (50, 50), (50, 100)\}$ with samples drawn IID from exponential and Gaussian distributions with varying parameters, as well as $(n_0, n_1) \in \{(100, 50), (50, 100)\}$ drawn IID from uniform distributions with different parameters. For each configuration, we replicate the experiment 5000 times, compute the optimal threshold $\theta$ in each instance, and report the resulting empirical density $\mathbb{P}(\theta)$. 

\begin{figure}[tbhp!]
	\centering
\includegraphics[width=0.7\textwidth]{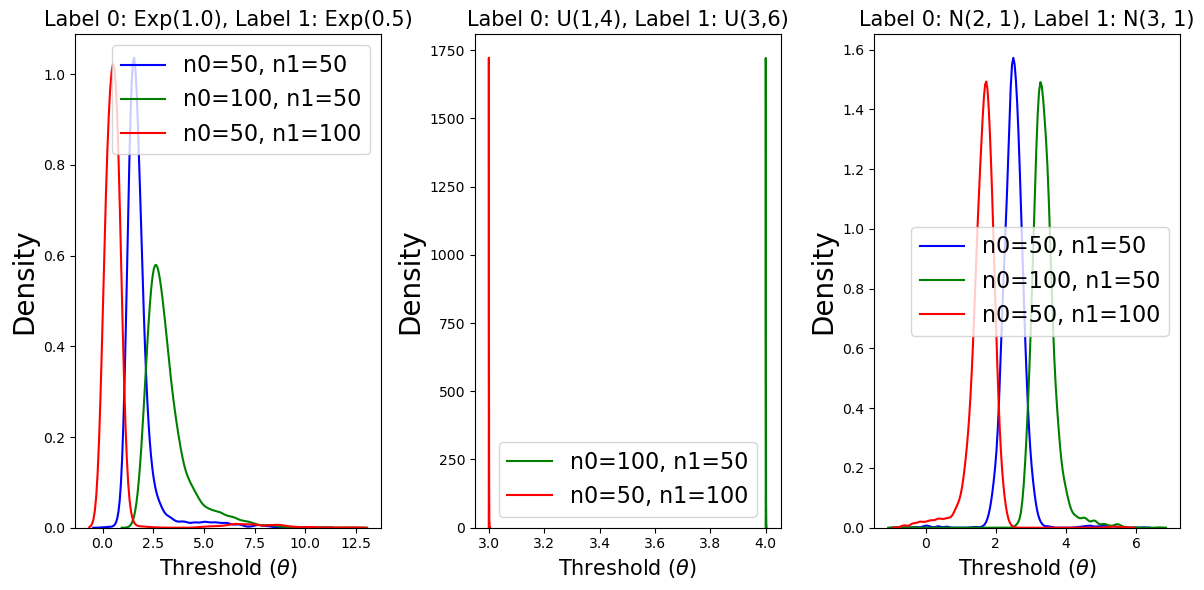}
	\caption{$\mathbb{P}(\theta)$ under different distributions with different sample size $n_0$ and $n_1$.}
	\label{fig:p_theta}
\end{figure}

\textbf{Special cases.} With $\theta$ being largely sensitive to the statistics of the initial training dataset, a universal closed-form expression for $\mathbb{P}(\theta)$ across all label distributions in not available. However, if some information about the feature-label distribution family is known, it may be possible to characterize this distribution. In particular, consider the case where $n_0$ label 0 samples are drawn IID from $U(a, a+c)$, and $n_1$ label 1 samples drawn IID from $U(b, b+c)$ with $a < b < a+c$, and $n_0 \neq n_1$. Theoretically, if $n_1 > n_0$, the optimal decision threshold will be $\hat{b}$, where $\hat{b}$ is the minimum value among all label 1 samples. Then, by the order statistics of the uniform distribution, we can find that  $\theta$ follows a distribution $b + 2Beta(1,n_1)$. Conversely, if $n_1 < n_0$, the optimal decision threshold will be $\hat{a}+c$, which follows a distribution $a+c + 2Beta(1,n_0)$.  

\textbf{Estimating the distribution.} In practice, one can envision obtaining the empirical density of $\theta$. For example, a bank's headquarters could collect decision thresholds $\theta$ from all its local branches, assuming each branch has had its own (slightly different) decision-making rule depending on the data available to them. In this case, provided sufficient samples, according to the DKW inequality, the empirical distribution could serve as a good estimation of the true $\mathbb{P}(\theta)$. 

\paragraph{Comparing the CDF error bounds between \emph{ex-post} and \emph{ex-ante} cases}

We illustrate our derived bounds (with $\delta = 0.015$) on the empirical CDF. We start with 50 random samples each from the label 0 distribution N(7,1) and label 1 distribution N(10,1). Then, 200 new samples are drawn from N(7,1). Throughout this experiment, the LB is set to 7. The decision threshold $\theta$ is determined from a collection of training samples from both labels. In the \emph{ex-post analysis} case, the decision threshold is found and fixed at 8.44. This means that all samples with features $x \geq \theta=8.44$ are accepted, and samples with features $LB=7 \leq x \leq \theta$ are accepted with a probability $\epsilon = 0.5$. 

\begin{figure}[ht]
	\centering
    \includegraphics[width=0.6\textwidth]{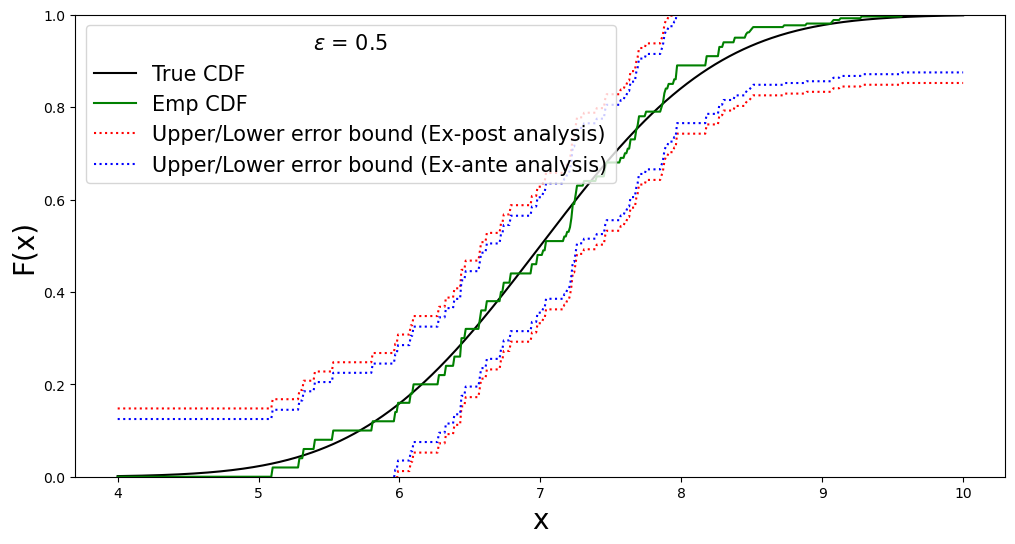}
    \caption{CDF error bounds evaluated under different types of randomness.}
	\label{fig:DKW_randomness_comparison}
\end{figure}

Similar to the results in Figure~\ref{fig:DKW}, we first note that our bounds (the dotted lines) effectively enclose the true distribution, evidenced by that the true distribution is upper and lower bounded by the dotted lines. We also note the empirical CDFs (solid lines) in the disclosed region are ``smoother'' compared to those in the censored region. Furthermore, we can also note that incorporating the dependence of $\theta$ on the random initial data can yield tighter bounds. However, this has required knowledge of the data distributions for both labels.

\section{Conclusion, Limitations and Future Work}\label{sec:conclusion}

We studied generalization error bounds for classification models learned from non-IID data collected under censored feedback. We presented two generalizations of the Dvoretzky-Kiefer-Wolfowitz (DKW) inequality, which characterizes the gap between empirical and theoretical CDFs given \emph{IID} data, to problems with \emph{non-IID} data due to censored feedback without exploration (Theorem~\ref{thm:two_subdomains}) and with exploration (Theorem~\ref{thm:three_subdomains}), and connected these bounds to generalization error guarantees of the learned model (Theorem~\ref{thm:error}). Our findings establish the extent to which a decision maker should be concerned about censored feedback's impact on the learned model's performance guarantees, and show that a minimum level of exploration is needed to alleviate it. Additionally, we analyze the bounds under the \emph{ex-ante analysis} setting in Section~\ref{sec:full_random} and show that it can yield tighter bounds compared to the \emph{ex-post analysis} case.

\textbf{Distribution-dependent bounds.} As discussed in Section~\ref{sec:full_random}, a key limitation in the \emph{ex-ante analysis} case is the need to know the label distribution to derive the prior probability $\mathbb{P}(\theta)$. While empirical estimation of $\theta$ may be feasible in practice, our generalization error bounds—derived from CDF error bounds—still inherently rely on the label distributions. An alternative approach would be to develop error bounds based on distribution-free complexity measures (e.g., VC dimension, Rademacher complexity, etc.), leveraging the richness of the hypothesis class. This remains a potential direction for future work.

\textbf{Other future work.} For future work, we are also interested in strengthening our bounds by allowing the model ($\theta$) to be adaptively updated as new samples are collected; as noted in Section~\ref{sec:numerical}, this could help further strengthen our error bounds. Generalization error bounds under a combination of censored feedback and domain adaptation are also worth exploring, wherein the initial training data distribution differs from the target domain distribution. Finally, we have provided extensions of the DKW inequality, which strengthens the VC inequality when data is real-valued, under censored feedback; providing similar extensions of the VC inequality for \emph{multi-dimensional data} could be an interesting direction of future work. We discuss some initial findings and potential challenges of this extension below. 

\textbf{Bounds for higher dimensional data}. When assessing generalization error under censored feedback in higher dimensional data, one approach could be to first reduce the dimensionality, enabling direct application of our findings. For instance, we have performed a mapping of multi-dimensional features to a single-dimensional representation in our experiments on the real-world \emph{Adult} census dataset. However, this reduction may lead to some loss of information, potentially impacting algorithm performance. An alternative would be to follow our approach of identifying IID subspaces in the higher-dimensional data space, apply a \emph{multivariate} DKW inequality (e.g., \citep{naaman2021tight}) in these subspaces, and then identify the appropriate error coefficients to re-assemble the subdomain bounds and find a CDF error bound for the entire data domain. We provide an analysis for 2D spaces based on this approach in Appendix~\ref{app:higher_dimension}. A main challenge when doing so is that while the decision boundary can be any arbitrary line (determining the two subspaces in which data can be viewed as IID), the standard joint CDF calculates the probability that  $X \leq x$ and $Y \leq y$, where $x$ and $y$ are vertical and horizontal cutoff values. To circumvent this mismatch, we start with an \emph{adjusted} CDF which measures data density and counts existing vs. newly collected samples in a ``rotated'' data space, and subsequently map the CDF error bound of the adjusted CDF to a CDF error bound for the standard CDF (as detailed in Appendix~\ref{app:higher_dimension}). Alternative error bounds that build on the VC inequality for multi-dimensional data (instead of multi-dimensional DKW inequalities), remain as a potential direction for future work. 

\begin{appendices}

\section{Additional and detailed related work}\label{app:additional_review}

\textbf{Censored feedback (Adaptive sampling)}. The problem of censored feedback has been extensively studied in the literature. Prior works, such as \citep{deshpande2018accurate, ensign2018runaway, wei2021decision}, have explored various aspects of this problem. For instance, \citet{ensign2018runaway} view the problem as a semi-bandit problem and apply existing regret bounds to obtain their theoretical result. \citet{deshpande2018accurate} investigate inference in high-dimensional linear models and analyze the sample adaptivity from a bias and variance trade-off perspective in both time series analysis and batched data collection. \citet{wei2021decision} studies the learning policies with censored feedback, finding the optimal decision policy that maximizes the utility function. This work considers both adaptive sampling and cost exploration and formulates the problem as a partially observed Markov decision process, which is solved through dynamic programming. Our work is similar to these prior works in that we also consider censored feedback, but we mainly focus on deriving a generalization error bound in the presence of censored feedback, and we propose to use a bounded exploration technique to improve the bounds.

\textbf{Generalization with non-IIDness.} In addition to the $\beta$-mixing sequence non-IIDness discussed in the main paper, \citep{steinwart2009fast,steinwart2009learning} have studied a more general non-IID setting called $\alpha$-mixing sequence, where stationarity is not necessary. Furthermore, \citep{modha1996minimum,zou2009generalization} have explored the generalization performance based on exponentially strongly mixing observations, where the $\alpha$-mixing coefficient is bounded by an exponential decay term. Apart from the mixing sequence, \citep{cheng2018simple} have imitated human learning behavior by discarding training samples with lower losses during training, and refreshing them when a new training epoch starts. They experimentally show that their method of using non-IID samples can lead to better performance than uniformly sampling using IID samples. \citep{smale2009online} have investigated non-IIDness in online learning algorithms through a Markov sampling method, where samples are drawn according to a sequence of probability measures. Meanwhile, \citep{zhao2022gray} explore non-IIDness through a mixing of in-distribution and out-of-distribution samples. They show that the performance gap is related to the distribution discrepancy between in-distribution and out-of-distribution (similar to \citep{ben2010theory}, which refers to source and target discrepancy), empirical losses, and empirical estimation error.

\textbf{DKW-typed inequalities with non-IIDness.} Our work is also related to DKW-type inequalities with non-IID data. \citep{kontorovich2014uniform} derive deviation bounds using DKW-type inequalities for strongly mixing non-stationary Markov chains. Extending the DKW inequality with one-dimensional IID samples sharpened by \citep{massart1990tight}, \citep{naaman2021tight} generalize it to the multi-dimensional case, incorporating an $\alpha$-mixing condition. Furthermore, DKW-type inequalities have been applied to the Kaplan-Meier (KM) estimator in the analysis of right censored data in survival analysis, as shown by \citep{bitouze1999dvoretzky}. They utilize a weight factor, similar to our scaling factor, to control uniform deviation in the uncensored data model. Moreover, \citep{goldberg2017support} introduce a censored SVM learning algorithm with a inverse probability censoring weighting (IPCW) loss function. Additionally, \citep{goldberg2019hoeffding} develope Hoeffding-type and Bernstein-type inequalities based on DKW inequalities to bound the difference between the IPCW estimator and its expectation. 

\textbf{Vapnik-Chevronenkis (VC) Theory.} In binary classification with the 0-1 loss function, the VC theory \citep[Thm 12.6]{pollard2012convergence} provides a bound on the maximum difference between the empirical estimates of whether an event happens and its expectation over IID data. It states that the difference could be bounded by $8S(\mathcal{F},n)\exp(-n\eta^2/32)$, where $S(\mathcal{F},n)$ is the shattering coefficient of $\mathcal{F}$ and $n$ is the sample size. The shattering coefficient satisfies $S(\mathcal{F},n) \leq (n+1)^d$ with finite VC dimension $d$ so that the generalization error of linear classifiers with one dimensional IID samples could be bounded as $8(n+1)^2\exp(-n\eta^2/32)$. Alternatively, as the VC theory is a mighty generalization of Gilvenko-Cantelli (GC) bound \citep[Thm 12.4]{pollard2012convergence}, when considering the deviation between the empirical and true CDFs, it takes the form of $8(n+1)\exp(-n\eta^2/32)$ by replacing $S(\mathcal{F},n)$ with $n+1$. In our study, we employ the DKW bound, which is an extension of the GC bound by eliminating the logarithmic term $(n+1)$. The main distinction between our bounds and the VC/GC bounds lies in the assumption that the data input is non-IID over the entire data domain, while it may be treated as IID within each subdomain. 

\section{Proofs}\label{app:all_proofs}

\subsection{Proof of Lemma~\ref{lemma:left}} \label{app:lemma_left}

\begin{proof}Let $G$ and $G_{m}$ be the theoretical and empirical CDFs of $Z$ in the censored region, respectively.

\noindent For the full c.d.f. in $(-\infty, \theta)$, we have 
\[F(x) = \mathbbm{P}(X \leq x), \hspace{0.2in} F_{n}(x) = \frac{1}{n}\sum_{i=1}^{n} \mathbf{1}_{X_i \leq x}, \hspace{0.1in} \forall i, x \in (-\infty, \theta)\]
For the censored c.d.f. in $(-\infty, \theta)$, we have
\[G(x) = \mathbbm{P}(Z \leq x), \hspace{0.2in} G_{m}(x) = \frac{1}{m}\sum_{i=1}^{m} \mathbf{1}_{Z_i \leq x}, \hspace{0.1in} \forall i, x \in (-\infty, \theta)\]
Hence, for any fixed $x \in (-\infty,\theta)$, we have
\[\sum_{i=1}^{m} \mathbf{1}_{Z_i \leq x} = \sum_{i=1}^{n} \mathbf{1}_{X_i \leq x} = A\]
and 
\[F_{n}(x) = \frac{1}{n}\sum_{i=1}^{n} \mathbf{1}_{X_i \leq x} = \frac{1}{n} A = \frac{m}{n} \cdot \frac{1}{m}A = \frac{m}{n} G_{m}(x)\]
Similarly, 
\[G(x) = \mathbbm{P}(Z \leq x) = \mathbbm{P}(X \leq x | X \leq \theta) = \frac{\mathbbm{P}(X \leq x)}{\mathbbm{P}(X \leq \theta)} = \frac{F(x)}{\alpha}\]
Therefore, we have 
\begin{align*}
    \sup_{x \in (-\infty, \theta)}&|F-F_{n}| = \sup_{x \in (-\infty, \theta)}\Big|\alpha G - \frac{m}{n} G_{m} \Big| \\
    &\qquad =  \sup_{x \in (-\infty, \theta)}\Big|\alpha G  - \alpha G_{m} + \alpha G_{m} - \frac{m}{n} G_{m} \Big|\\
    &\qquad \leq \sup_{x \in (-\infty, \theta)}\Big|\alpha (G  - G_{m})\Big| + \sup_{x \in (-\infty, \theta)}\Big|(\alpha - \frac{m}{n}) G_{m} \Big|\\
    &\qquad = \sup_{x \in (-\infty, \theta)}\Big|\alpha (G  - G_{m})\Big| + \Big|\alpha - \frac{m}{n}\Big|
\end{align*}
where the last equality holds since $|\alpha - \frac{m}{n}|$ is constant once $\theta$ is fixed and by the monotonic property of the CDF curve. Notice that, by the definition of $G_{m} = \frac{1}{m}\sum_{i} \mathbf{1}_{Z_i \leq \theta} = 1$. We could instead choose to add and subtract $\frac{m}{n} G$, which would yield a similar result. Choosing the smaller of the two gives us the bound.
\end{proof}

\subsection{Proof of Lemma~\ref{lemma:right}} \label{app:lemma_right}

\begin{proof}
Let $K$ and $K_{n-m+k}$ be the theoretical and empirical CDFs of $Z$ in the disclosed region, respectively.
\[K(x) = \mathbb{P}(Z\leq x), \hspace{0.1in} K_{n-m+k}(x) = \frac{1}{n-m+k}\sum_{i=1}^{n-m+k} \mathbbm{1}(Z_i \leq x), \hspace{0.1in} \forall x \in (\theta, \infty).\]

Since $Z = \{X|X \geq \theta\}$, then we have 
\[K(x) = \mathbbm{P}(Z \leq x) = \mathbbm{P}(X \leq x | X \geq \theta) = \frac{\mathbbm{P}(\theta \leq X \leq x)}{\mathbbm{P}(X \geq \theta)} = \frac{F(x)-\alpha}{1-\alpha}\]
In terms of the empirical CDF, the full empirical CDF on variable $X$ in $(\theta, \infty)$ can be viewed as starting at the $F_n(\theta)=\frac{m}{n}$, then increasing to 1 as $x$ increases with the total increment $\frac{n-m}{n}$. However, for the increment, the only difference between $F_{n+k}$ and $K_{n-m+k}$ is that the earlier one focuses on the variable $X$ while the latter one focuses on the variable $Z = \{X|X\geq \theta\}$. Note that, the notation $F_{n+k}$, similar to $F_n$, means the empirical CDF is evaluated using $n+k$ samples. Hence, with additional $k$ samples, we can use a scaling factor $\frac{n-m}{n}$ to write the CDF such that 
\[F_{n+k}(x) = \frac{m}{n} +  \frac{n-m}{n}\frac{1}{n-m+k}\sum_{i=1}^{n-m+k} \mathbf{1}_{\theta \leq X_i \leq x} , \hspace{0.1in} \forall i, x \in (\theta, \infty)\]
Then, for any fixed $x \in (\theta, \infty)$, we have
\[F_{n+k}(x) - \frac{m}{n} = \frac{n-m}{n}\frac{1}{n-m+k}\sum_{i=1}^{n-m+k} \mathbf{1}_{\theta \leq X_i \leq x} = \frac{n-m}{n} K_{n-m+k}(x)\]
Therefore, we have 
\begin{align*}
    \sup_{x \in (\theta, \infty)}&|F-F_{n+k}| = \sup_{x \in (\theta, \infty)}\Big|(1-\alpha) K + \alpha - \frac{n-m}{n} K_{n-m+k} - \frac{m}{n}\Big| \\
    &\qquad \hspace{-0.5in} \leq \sup_{x \in (\theta, \infty)}\Big|(1-\alpha) K - \frac{n-m}{n} K_{n-m+k}\Big| + \Big|\alpha - \frac{m}{n}\Big|  \\
    &\qquad \hspace{-0.5in} =  \sup_{x \in (\theta, \infty)}\Big|(1-\alpha) K  - (1-\alpha) K_{n-m+k} + (1-\alpha) K_{n-m+k} - \frac{n-m}{n} K_{n-m+k}\Big| + \Big|\alpha - \frac{m}{n}\Big|\\
    &\qquad \hspace{-0.5in} \leq \sup_{x \in (\theta, \infty)}\Big|(1-\alpha) (K  - K_{n-m+k})\Big| + 2\Big|\alpha - \frac{m}{n}\Big|
\end{align*}
Similarly, we could instead choose to add and subtract $\frac{n-m}{n}K$ and choose the smaller of the two.  
\end{proof}

\subsection{Proof of Theorem~\ref{thm:two_subdomains}} \label{app:thm2}

\begin{proof}
\begin{align*}
    &\mathbb{P}\bigg(\sup_{x\in R} \Big|F(x) - F_{n+k}(x)\Big|\geq \eta \bigg) \\
    &\qquad \hspace{-0.2in} = \mathbb{P}\bigg(\max\bigg(\sup_{x\in (-\infty, \theta)} \Big|F(x) - F_{n}(x)\Big|, \sup_{x\in (\theta, \infty)} \Big|F(x) - F_{n+k}(x)\Big|\bigg)\geq \eta \bigg)\\
    &\qquad \hspace{-0.2in} = \mathbb{P}\bigg(\sup_{x\in (-\infty, \theta)} \Big|F(x) - F_{n}(x)\Big| \geq \eta ~~ \text{ or } \sup_{x\in (\theta, \infty)} \Big|F(x) - F_{n+k}(x)\Big|\geq \eta \bigg)\\
    &\qquad \hspace{-0.2in} \leq  \mathbb{P}\bigg(\sup_{x\in (-\infty, \theta)} \Big|F(x) - F_{n}(x)\Big| \geq \eta \bigg) + \mathbb{P}\bigg(\sup_{x\in (\theta, \infty)} \Big|F(x) - F_{n+k}(x)\Big|\geq \eta  \bigg)  \\
    &\qquad \hspace{-0.2in} \leq \mathbb{P}\bigg(\sup_{x \in (-\infty, \theta)}\Big|\min\Big(\alpha, \frac{m}{n}\Big) (G  - G_{m})\Big| + \Big|\alpha - \frac{m}{n} \Big| \geq \eta \bigg)  \\
    &\qquad \hspace{0.2in} + \mathbb{P}\bigg(\sup_{x \in (\theta, \infty)}\Big|\min\Big(1-\alpha,\frac{n-m}{n}\Big) (K  - K_{n-m+k})\Big| + 2\Big|\alpha - \frac{m}{n} \Big| \geq \eta \bigg) \allowdisplaybreaks\\
    &\qquad \hspace{-0.2in} = \mathbb{P}\bigg(\sup_{x \in (-\infty, \theta)}\Big|G  - G_{m}\Big| \geq \frac{\eta-\Big|\alpha - \frac{m}{n} \Big| }{\min\Big(\alpha, \frac{m}{n}\Big)} \bigg) \\
    &\qquad \hspace{0.2in} +  \mathbb{P}\bigg(\sup_{x \in (\theta, \infty)}\Big|K - K_{n-m+k}\Big| \geq \frac{\eta -2\Big|\alpha - \frac{m}{n} \Big|}{\min\Big(1-\alpha,\frac{n-m}{n}\Big)} \bigg) \\
    &\qquad \hspace{-0.2in} \leq 2\exp\Big({\tfrac{-2m(\eta-|\alpha - \frac{m}{n}|)^2}{\min\big(\alpha, \frac{m}{n}\big)^2}}\Big) + 2\exp\Big({\tfrac{-2(n-m+k)(\eta-2|\alpha - \frac{m}{n}|)^2}{\min\big(1-\alpha,\frac{n-m}{n}\big)^2}}\Big)
    \end{align*}

where the first inequality is obtained by the union bound, the second inequality is obtained according to Lemma ~\ref{lemma:left} and 
 \ref{lemma:right}, and the last inequality is obtained by plugging in the DKW inequality introduced in Theorem ~\ref{thm:GC_theorem}. 
\end{proof}

\subsection{Proof of Theorem~\ref{thm:three_subdomains}} \label{app:thm3}

\begin{proof}
\begin{align*}
    &\mathbb{P}\bigg(\sup_{x\in R} \Big|F(x) - F_{n+k_e+k_d}(x)\Big|\geq \eta \bigg) \\ 
    &\qquad \hspace{-0.2in} = \mathbb{P}\bigg(\sup_{x\in (-\infty, \text{LB})} \Big|F(x) - F_{n}(x)\Big| \geq \eta \text{ or}
    \sup_{x\in (\text{LB},\theta)} \Big|F(x) - F_{n+k_e}(x)\Big| \geq \eta\\
     &\qquad \hspace{2.8in}\text{ or} \sup_{x\in (\theta, \infty)} \Big|F(x) - F_{n+k_e+k_d}(x)\Big|\geq \eta \bigg)\\
    &\qquad \hspace{-0.2in} \leq \mathbb{P}\bigg(\sup_{x \in (-\infty, \text{LB})}\Big|\min\Big(\beta, \frac{l}{n}\Big) (G  - G_{l})\Big| + \Big|\beta - \frac{l}{n} \Big| \geq \eta \bigg) \allowdisplaybreaks \\
    &\qquad  + \mathbb{P}\bigg(\sup_{x \in (\text{LB}, \theta)}\Big|\min\big(\alpha-\beta, \tfrac{n-l}{n}\tfrac{m-l+k_e}{n-l+k_e+\epsilon k_d}\big)(E  - E_{m-l+k_e})\Big| \\
    &\qquad \hspace{2.4in} + \Big|\alpha-\beta - \tfrac{n-l}{n}\tfrac{m-l+k_e}{n-l+k_e+\epsilon k_d} \Big|+\Big|\beta - \frac{l}{n} \Big| \geq \eta \bigg)\\
    &\qquad  + \mathbb{P}\bigg(\sup_{x \in (\theta, \infty)}\Big|\min\Big(1-\alpha, \tfrac{n-l}{n}\tfrac{n-m+\epsilon k_d}{n-l+k_e+\epsilon k_d}\Big)(K  - K_{n-m+k_d})\Big| \\
    &\qquad \hspace{2.4in} + 2\Big|\alpha - \tfrac{l}{n} -  \tfrac{n-l}{n}\tfrac{m-l+k_e}{n-l+k_e+\epsilon k_d} \Big| \geq \eta \bigg)  \\ 
    &\qquad \hspace{-0.2in} = \mathbb{P}\bigg(\sup_{x \in (-\infty, \text{LB})}\Big|G  - G_{l}\Big| \geq \frac{\eta-\Big|\beta - \frac{l}{n} \Big| }{\min\Big(\beta, \frac{l}{n}\Big)}\bigg) \\
    &\qquad + \mathbb{P}\bigg(\sup_{x \in (\text{LB}, \theta)}\Big|E - E_{m-l+k_e}\Big| \geq \frac{\eta -\Big|\beta - \frac{l}{n}\Big| - \Big|\alpha -\beta - \frac{n-l}{n}\frac{m-l+k_e}{n-l+k_e+\epsilon k_d}\Big|}{\min\Big(\alpha-\beta,\frac{n-l}{n}\frac{m-l+k_e}{n-l+k_e+\epsilon k_d}\Big)} \bigg)  \\
    &\qquad + \mathbb{P}\bigg(\sup_{x \in (\theta, \infty)}\Big|K - K_{n-m+k_d}\Big| \geq \frac{\eta -2\Big|\alpha - \frac{l}{n} -  \frac{n-l}{n}\frac{m-l+k_e}{n-l+k_e+\epsilon k_d} \Big|}{\min\Big(1-\alpha,\frac{n-l}{n}\frac{n-m+\epsilon k_d}{n-l+k_e+\epsilon k_d}\Big)} \bigg) \\
    &\qquad \hspace{-0.2in} \leq 2\exp({\tfrac{-2l(\eta-|\beta - \frac{l}{n} |)^2}{\min\Big(\beta, \frac{l}{n}\Big)^2}}) \\
    &\qquad  + 2\exp({\tfrac{-2(m-l+k_e)\Big(\eta-|\beta - \frac{l}{n}| - \Big|\alpha -\beta - \frac{n-l}{n}\frac{m-l+k_e}{n-l+k_e+\epsilon k_d}\Big|\Big)^2}{\min\Big(\alpha-\beta,\frac{n-l}{n}\frac{m-l+k_e}{n-l+k_e+\epsilon k_d}\Big)^2}})\\
    &\qquad + 2\exp({\tfrac{-2(n-m+k_d)(\eta-2|\alpha - \frac{l}{n} -  \frac{n-l}{n}\frac{m-l+k_e}{n-l+k_e+\epsilon k_d}|)^2}{\min\Big(1-\alpha,\frac{n-l}{n}\frac{n-m+\epsilon k_d}{n-l+k_e+\epsilon k_d}\Big)^2}})
    \end{align*}

The proof is similar to the proof of Theorem~\ref{thm:two_subdomains}, with the introduction of exploration region and re-estimated empirical estimates $\alpha$. 
\end{proof}

\subsection{Proof of Proposition~\ref{prop:B(LB) < B(theta)}} \label{app:prop1}

\begin{proof}
Before we dive into the proof of the proposition, we first show that the impact of the estimation error under different $\theta$ is minimal compared to the change of $\theta$. 

Let's consider the censored region (first term) in the above inequality, and w.l.o.g., we assume that $m/n < \alpha$ and $\alpha - m/n = u$. Then, the first term could be written as 
\[2\exp\Big(\tfrac{-2m(\eta-u)^2}{(\frac{m}{n})^2}\Big) = 2\exp\Big(\tfrac{-2n^2(\eta-u)^2}{m}\Big)\]
As $\theta$ decreases, $m$ decreases. It is worth noting that $E(u)=0$ according to the law of large numbers. In the following analysis, we measure the change as factor $b, b'$ for $m, u$, respectively, and denote $\Delta_m, \Delta_u$ as the change in the output. We also assume $0 \leq b, b' \leq 1$, and use it as a proxy to show the decreasing in $\theta$. In other words, we are analyzing the impact of change in $u$ compared to the change in $m$ as $\theta$ decreases, our objective is to show the following 
\[\frac{\Delta_m}{\Delta_u}=\frac{2\exp\Big(\tfrac{-2n^2(\eta-u)^2}{m}\Big) - 2\exp\Big(\tfrac{-2n^2(\eta-u)^2}{bm}\Big)}{2\exp\Big(\tfrac{-2n^2(\eta-u)^2}{m}\Big) -2\exp\Big(\tfrac{-2n^2(\eta-b'u)^2}{m}\Big)} > 1\]
which is equivalent to show \[\frac{ \exp\Big(\tfrac{-2n^2(\eta-u)^2}{bm}\Big)}{\exp\Big(\tfrac{-2n^2(\eta-b'u)^2}{m}\Big)} < 1\]
Since we have, 
\begin{align*}
    \frac{ \exp\Big(\frac{-2n^2(\eta-u)^2}{bm}\Big)}{\exp\Big(\frac{-2n^2(\eta-b'u)^2}{m}\Big)} &= \exp(\frac{-2n^2(\eta-u)^2}{bm}- \frac{-2n^2(\eta-b'u)^2}{m})\\
    &\qquad = \exp(\frac{-\frac{2}{b}n^2(\eta-u)^2 + 2n^2(\eta-b'u)^2}{m})\\
    &\qquad = (\exp(-\frac{2n^2}{m}))^{(\frac{1}{b}(\eta-u)^2 - (\eta - b'u)^2)} \\
\end{align*}
It is clear that the exponent is greater than 0 since the change magnitude of $m$ could be much greater than that of $u$ as $\theta$ decreases, meaning that $0<b<<b'\approx 1$. Therefore, the impact of the estimation error under different $\theta$ is minimal compared to the change of $\theta$. As a result, the first term decreases as $\theta$ decreases, primarily due to the decrease in the denominator. 

Regarding the disclosed region (second term), the sample size increases due to a reduction in the number of samples ($m$) in the censored region and the additional samples ($k$) collected from the disclosed region. Similar to the discussion in the proof of Proposition~\ref{prop:B(LB) < B(theta)}, let $k = c(n-m)$, where $c$ is a scalar representing how many additional samples we could collect proportionally above the decision threshold. If $c \geq \frac{(n-m)(\eta-u)^2}{m(\eta-2u)^2}-1$, as more additional samples are collected, it will contribute to making the second term vanish and become dominated by the first term. Consequently, as $\theta$ decreases, the bounds also decrease, leading to $B(\theta)$ decreases as $\theta$ decreases.
\end{proof}

\subsection{Proof of Proposition~\ref{prop:B(LB, theta, 1) < B(LB, theta, epsilon)}} \label{app:prop2}

\begin{proof}
As $\epsilon$ decreases, the censored region (first term) remains unchanged since the change in $\epsilon$ only affects the value of $k_1$. For the disclosed region (third term), similar to Proposition~\ref{prop:B(LB) < B(theta)}, the impact of the estimation error $|\alpha - \frac{l}{n} -  \frac{n-l}{n}\frac{m-l+k_e}{n-l+k_e+\epsilon k_d}|$ under different values of $\epsilon$ is minimal compared to the change in $k_e$. Therefore, the whole expression is dominated by the change in the exploration region (second term).

For the exploration region, it is worth noting that when $\epsilon = 1$, we collect $k_e$ samples in the exploration region. Therefore, if $\epsilon < 1$, the number of $k_e$ will decrease leading to an increase in the exponential term. Thus, as $\epsilon$ becomes smaller, the bound in the exploration region increases, resulting in an overall increase in the bounds. This implies that $B^e(LB, \theta, \epsilon)$ is decreasing in $\epsilon$.
\end{proof}

\subsection{Proof of Theorem~\ref{thm:error}} \label{app:thm4}

\begin{proof}
Let $F^1_{n_1+k_1}(\theta):= \frac{1}{n_1}\sum_{(x_i,y_i)} \mathbbm{1}{\{x_i \leq \theta,y_i = 1\}}$ and $F^0_{n_0+k_0}(\theta):= \frac{1}{n_0}\sum_{(x_i,y_i)} \mathbbm{1}{\{x_i \leq \theta,y_i = 0\}}$ be the empirical estimate of the area of the censored region for the corresponding label. 

\begin{align*}
    &\Big|R(\hat{\theta}) - R_{emp}(\hat{\theta})\Big| \leq \sup_{\theta} \Big|R(\theta) - R_{emp}(\theta)\Big|\\
    &\qquad = \sup_{\theta} \Big|p_1F^1(\theta) + p_0\Big(1-F^0(\theta)\Big) - \frac{n_1}{n}F^1_{n_1+k_1}(\theta)- \frac{n_0}{n}\Big(1-F^0_{n_0+k_0}(\theta)\Big)\Big|\\
    &\qquad \leq \sup_{\theta} \Big|p_1F^1(\theta) - \frac{n_1}{n}F^1_{n_1+k_1}(\theta)\Big| + \sup_{\theta} \Big|p_0F^0(\theta) - \frac{n_0}{n}F^0_{n_0+k_0}(\theta)\Big| +|p_0 - \frac{n_0}{n}|\\
    &\qquad \leq \sum_{y \in \{0,1\}}\min\Big(p_y, \frac{n_y}{n}\Big)\sup_{\theta}\Big|F^y(\theta) - F^y_{n_y+k_y}(\theta)\Big| + 3|p_0 - \frac{n_0}{n}|
\end{align*}
The last inequality follows the same technique as the proof of Lemma~\ref{app:lemma_left}.

\end{proof}

\section{Ex-ante analysis bounds with exploration }\label{app:partial_full_exploration_bounds}

\textbf{Bounds with exploration under the \emph{ex-ante analysis} case:}
\begin{corollary}\label{thm:three_subdomains_full}
Let $x_1, x_2, \ldots, x_n$ be random initial data samples, drawn IID from \emph{single} distribution with CDF $F(x)$. Let the optimal $\theta$, derived from a collection of initial data samples with \emph{both} labels, and $LB$ partition the domain into three regions, such that $\beta=F(LB)$ and $\alpha=F(\theta)$, with $l$ and $m$ of the initial $n$ samples located to the left of $LB$ and $\theta$, respectively. Assume we have collected an additional $k_e$ samples between $LB$ and $\theta$, under an exploration probability $\epsilon$, and an additional number of $k_d$ samples above $\theta$. Let $F_{n+k_e+k_2}(x)$ denote the empirical CDF estimated from these $n+k_e+k_d$ non-IID samples. Then, for every $\eta, k_e, k_d >0$,
{
\begin{align*}
    &\mathbb{P}\bigg(\sup_{x\in R} \Big|F(x) - F_{n+k_e+k_d}(x)\Big|\geq \eta \bigg) \\ 
    &\qquad \hspace{-0.2in} \leq \int_{-\infty}^{\infty} \Bigg [\sum_{m=1}^{n-1}\sum_{l=1}^{m-1} \Bigg(2\exp({\tfrac{-2l(\eta-|\beta - \frac{l}{n} |)^2}{\min\Big(\beta, \frac{l}{n}\Big)^2}}) \\
    &\qquad  \hspace{1in}+ 2\exp({\tfrac{-2(m-l+k_e)\Big(\eta-|\beta - \frac{l}{n}| - \Big|\alpha -\beta - \frac{n-l}{n}\frac{m-l+k_e}{n-l+k_e+\epsilon k_d}\Big|\Big)^2}{\min\Big(\alpha-\beta,\frac{n-l}{n}\frac{m-l+k_e}{n-l+k_e+\epsilon k_d}\Big)^2}})\\
    &\qquad \hspace{1in} + 2\exp({\tfrac{-2(n-m+k_d)(\eta-2|\alpha - \frac{l}{n} -  \frac{n-l}{n}\frac{m-l+k_e}{n-l+k_e+\epsilon k_d}|)^2}{\min\Big(1-\alpha,\frac{n-l}{n}\frac{n-m+\epsilon k_d}{n-l+k_e+\epsilon k_d}\Big)^2}})\Bigg)\mathbb{P}(l,m|\theta) \\
    &\qquad \hspace{-0.2in}+ \sum_{m=1}^{n-1}\Bigg(1 + \exp({\tfrac{-2(m+k_e)\Big(\eta-\beta - \Big|\alpha -\beta - \frac{m+k_e}{n-l+k_e+\epsilon k_d}\Big|\Big)^2}{\min\Big(\alpha-\beta,\frac{m+k_e}{n-l+k_e+\epsilon k_d}\Big)^2}}) \\
    &\qquad \hspace{1in} +\exp({\tfrac{-2(n-m+k_d)(\eta-2|\alpha - \frac{m+k_e}{n+k_e+\epsilon k_d}|)^2}{\min\Big(1-\alpha,\frac{n-m+\epsilon k_d}{n+k_e+\epsilon k_d}\Big)^2}})\Bigg)\mathbb{P}(l=0|\theta)  \\
    &\qquad \hspace{-0.2in}+ \sum_{m=1}^{n-1}\Bigg(\exp({\tfrac{-2m(\eta-|\beta - \frac{m}{n} |)^2}{\min\Big(\beta, \frac{m}{n}\Big)^2}})+\exp({\tfrac{-2k_e\Big(\eta-|\beta - \frac{m}{n}| - \Big|\alpha -\beta - \frac{n-m}{n}\frac{k_e}{n-m+k_e+\epsilon k_d}\Big|\Big)^2}{\min\Big(\alpha-\beta,\frac{n-m}{n}\frac{k_e}{n-m+k_e+\epsilon k_d}\Big)^2}})\\ 
    &\qquad \hspace{1in}+\exp({\tfrac{-2(n-m+k_d)(\eta-2|\alpha - \frac{m}{n} -  \frac{n-m}{n}\frac{k_e}{n-m+k_e+\epsilon k_d}|)^2}{\min\Big(1-\alpha,\frac{n-m}{n}\frac{n-m+\epsilon k_d}{n-m+k_e+\epsilon k_d}\Big)^2}})\Bigg)\mathbb{P}(l=m|\theta) \allowdisplaybreaks \\
    &\qquad + \sum_{l=1}^{m-1} \Bigg (\exp({\tfrac{-2l(\eta-|\beta - \frac{l}{n} |)^2}{\min\Big(\beta, \frac{l}{n}\Big)^2}})+\exp({\tfrac{-2(n-l+k_e)\Big(\eta-|\beta - \frac{l}{n}| - \Big|\alpha -\beta - \frac{n-l}{n}\frac{n-l+k_e}{n-l+k_e+\epsilon k_d}\Big|\Big)^2}{\min\Big(\alpha-\beta,\frac{n-l}{n}\frac{n-l+k_e}{n-l+k_e+\epsilon k_d}\Big)^2}}) \\
    &\qquad \hspace{1in}+ \exp({\tfrac{-2k_d(\eta-2|\alpha - \frac{l}{n} -  \frac{n-l}{n}\frac{n-l+k_e}{n-l+k_e+\epsilon k_d}|)^2}{\min\Big(1-\alpha,\frac{n-l}{n}\frac{\epsilon k_d}{n-l+k_e+\epsilon k_d}\Big)^2}})\Bigg)\mathbb{P}(m=n|\theta)\\
    &\qquad + 2\Bigg(1+\exp({\tfrac{-2k_e\Big(\eta-\beta - \Big|\alpha -\beta - \frac{k_e}{n+k_e+\epsilon k_d}\Big|\Big)^2}{\min\Big(\alpha-\beta,\frac{k_e}{n+k_e+\epsilon k_d}\Big)^2}}) \\
    &\qquad \hspace{1in}+ \exp({\tfrac{-2(n-k_d)(\eta-2|\alpha - \frac{k_e}{n+k_e+\epsilon k_d}|)^2}{\min\Big(1-\alpha,\frac{n+\epsilon k_d}{n+k_e+\epsilon k_d}\Big)^2}})\Bigg)\mathbb{P}(l=0,m=0|\theta)\allowdisplaybreaks \\
    &\qquad + 2\Bigg(\exp({\tfrac{-2n(\eta-|\beta - 1|)^2}{\beta^2}})+1 + 1\Bigg)\mathbb{P}(l=n,m=n|\theta)\\
    &\qquad + 2\Bigg(1+\exp({\tfrac{-2(n+k_e)\Big(\eta-\beta - \Big|\alpha -\beta - \frac{n+k_e}{n+k_e+\epsilon k_d}\Big|\Big)^2}{\min\Big(\alpha-\beta,\frac{n+k_e}{n+k_e+\epsilon k_d}\Big)^2}}) \\
    &\qquad \hspace{1in}+ \exp({\tfrac{-2k_d(\eta-2|\alpha - \frac{n+k_e}{n+k_e+\epsilon k_d}|)^2}{\min\Big(1-\alpha,\frac{\epsilon k_d}{n+k_e+\epsilon k_d}\Big)^2}})\Bigg) \mathbb{P}(l=0,m=n|\theta)\Bigg ]\mathbb{P}(\theta)d\theta
    \end{align*}}
\end{corollary}

Similar to the Theorem~\ref{thm:two_subdomains_full}, this \emph{ex-ante analysis} case considers the decision threshold $\theta$ as a data-dependent random variable and applies the law of total probability by conditioning on the realization of $\theta$, $l$ of samples in the censored region, and $m-l$ of samples in the exploration region. 
 
\section{Expected and empirical loss derivation} \label{app:loss_derivation}

Following the setup described in Section~\ref{sec:model}, let $F^y$ and $F_n^y$ denote the CDFs and empirical CDFs, respectively, with $n$ realizations of the distribution of label $y \in \{0,1\}$. We assume that among these $n$ samples, $n_1$ of them belong to qualified agents ($y=1$), while the remaining ones are from unqualified agents ($y=0$). We suppose the model is equipped with a 0-1 loss function and a threshold-based binary classifier $f_{\theta}(x):\mathcal{X}\rightarrow \{0,1\}$. We then can derive the expected and empirical loss as follows:
\begin{align*}
    R(\theta)&=E_{XY}l(f(X),Y)\\
    &\qquad=\mathbbm{P}(Y=1)\cdot E_{X|Y=1}l(f(X),Y=1) + \mathbbm{P}(Y=0)\cdot E_{X|Y=0}l(f(X),Y=0)\\
    &\qquad = p_1 E\Big[\mathbbm{1}\{f(X) \neq Y, Y=1\}\Big] + p_0 E\Big[\mathbbm{1}\{f(X) \neq Y, Y=0\}\Big]\\
    &\qquad = p_1 \mathbbm{P}\Big(f(X) \leq \theta, Y=1\Big) + p_0 \mathbbm{P}\Big(f(X) \geq \theta, Y=0\Big)\\
    &\qquad = p_1F^1(\theta) + p_0(1-F^0(\theta))
\end{align*}
\begin{align*}
    R_{emp}(\theta) &=\frac{1}{n}\sum_{(x_i,y_i)} l(f(x_i),y_i) = \frac{1}{n}\sum_{(x_i,y_i)} \mathbbm{1}\{f(x_i) \neq y_i\}\\
    &\qquad = \frac{1}{n}\Big[\sum_{y_i = 1} \mathbbm{1}\{f(x_i) \neq 1\} + \sum_{y_i = 0} \mathbbm{1}\{f(x_i) \neq 0\}\Big]\\
    &\qquad =\frac{n_1}{n}\frac{1}{n_1}\sum_{(x_i,y_i)} \mathbbm{1}{\{x_i \leq \theta,y_i = 1\}} + \frac{n_0}{n}\Big(1-\frac{1}{n_0}\sum_{(x_i,y_i)} \mathbbm{1}{\{x_i \leq \theta,y_i = 0\}}\Big)
\end{align*}

\section{Analysis with higher dimensional samples} \label{app:higher_dimension}

In this section, we expand our analysis from real-valued data to a higher dimensional space using the multivariate DKW inequality introduced in \citep{naaman2021tight}. We first state the inequality and subsequently show how we will split the data domain to accommodate censored feedback. Although this analysis applies to samples with $k>1$ dimensions, for the sake of notation simplicity, we focus on samples with two dimensions.  

\begin{thm}[{Multivariate DKW inequality \citep{naaman2021tight}}]\label{thm:multivariate_dkw}
Let $Z_1, \ldots, Z_n$ be $k$ dimensional IID random variables with the empirical distribution function be $F_n(z) = \frac{1}{n}\sum_{i=1}^{n}\mathbbm{1}(Z_i\leq z)$. Let $F(z)$ be the expectation of $F_n(z)$. Then, for every $n$ and $\eta {> 0}$,
\[\mathbb{P}\bigg(\sup_{z\in \mathbb{R}^k} \Big|F(z) - F_n(z)\Big|\geq \eta \bigg) \leq 2k\exp{(-2n\eta^2)}~.\]
\end{thm}

In the 2D space, the threshold classifier is a linear classifier, with the decision boundary (line) represented as ($\mathbf{W}^T\mathbf{X}-b = 0$). Here, the bold terms denote vector forms, where $\mathbf{W}^T = [w_1, w_2]$ and $\mathbf{X}^T = [X_1, X_2]$. This line splits the full data domain ($\mathbb{R}^2$) into two subspaces: the subspace where $\mathbf{W}^T\mathbf{X}-b < 0$ is the censored region, while the remaining half space is the disclosed region. 

We first establish the relation between the deviation of the censored partial empirical CDF from its expectation, and the deviation of the full empirical CDF from its expectation. The proof mirrors that of Lemma ~\ref{lemma:left}.  

\begin{lemma}[Multivariate Censored Region]\label{lemma:multivariate_left}
Let $\textbf{Z} = \{\textbf{X} |\mathbf{W}^T\mathbf{X}-b < 0\}$ denote the $m$ out of $n+k$ samples that are in the censored region ($C$). Let $\alpha = F(\mathbf{W}^T\mathbf{X}-b < 0)$. Let $G$ and $G_{m}$ be the theoretical and empirical CDFs of $\textbf{Z}$, respectively. Then, 
{\begin{align*}
    \sup_{\mathbf{x} \in C}|F(\mathbf{x})-F_{n+k}(\mathbf{x})| \leq  \sup_{\mathbf{x} \in C}\underbrace{\Big|\min\Big(\alpha, \frac{m}{n}\Big)(G(\mathbf{x})  - G_{m}(\mathbf{x}))\Big|}_{\text{scaled censored subdomain error}} + \underbrace{\Big|\alpha - \frac{m}{n} \Big|}_{\text{scaling error}} .
\end{align*}}
\end{lemma}

In the disclosed region, we could derive the expression following a similar approach. However, in 2D space, for any given $\textbf{x}$ in the disclosed region, the number of samples in the disclosed region enclosed by $F(\textbf{x})$ should be no less than $n+k-m$ because we only need to subtract the number of samples from the area of $\mathbf{W}^T\mathbf{X}-b < 0$, $X_1 \leq x_1$ and $X_2 \leq x_2$ out of $m$. Since we only know there are a total of $m$ samples located in the censored region, we introduce an adjusted CDF $\Tilde{F}$, compared to the standard CDF $F(\textbf{x}) = Pr(X_1 \leq x_1, X_2 \leq x_2)$.

\begin{definition}[Adjusted CDF]
Let the decision line $\mathbf{W}^T\mathbf{X}-b = 0$ split the entire data domain into two subspaces. Let the adjusted CDF be 
\[\tilde{F}(x_1, b') = \int_{-\infty}^{\infty}\int_{-\infty}^{-\frac{w_1}{w_2}x_1 + \frac{w_1}{w_2}b'} f_{x_1,x_2} d_{x_1}d_{x_2}\]
where $w_1, w_2$ are the parameters defined by the line, and $b' \in (-\infty, \infty)$. 
\end{definition}

\begin{remark}
The standard CDF $F$ calculates the area of ($X_1 \leq x_1, X_2 \leq x_2$), where $x_1$ and $x_2$ are horizontal and vertical thresholds, respectively. In contrast, the adjusted CDF measures the area below the line $\mathbf{W}^T\mathbf{X}=b'$. When $b' = -\infty (\text{resp. } \infty)$, it is equivalent to the case of ($x_1 = x_2 = -\infty (\text{resp. } \infty)$). Similarly, when $b' = b$, $\tilde{F}(x_1, b'=b) = F(\mathbf{W}^T\mathbf{X}-b < 0) = \alpha$. The adjusted CDF and standard CDF are both non-decreasing functions ranging from 0 to 1, and they only differ in the ways of calculating the enclosed area. By using adjusted CDFs as a proxy, we can derive such relations in the disclosed region.
\end{remark}
  
\begin{lemma}[Multivariate Disclosed Region]\label{lemma:multivariate_right}
Let $\textbf{Z} = \{\textbf{X} |\mathbf{W}^T\mathbf{X}-b > 0\}$ denote the $n-m+k$ out of the $n+k$ samples in the disclosed region (D). Let $\alpha = F(\mathbf{W}^T\mathbf{X}-b < 0)$. Let $\tilde{K}$ and $\tilde{K}_{n-m+k}$ be the adjusted theoretical and empirical CDFs of $\textbf{Z}$, respectively. Then, 
{{\begin{align*} 
    \sup_{\mathbf{x} \in D}|F(\mathbf{x})-F_{n+k}(\mathbf{x})| &\leq  \sup_{\mathbf{x} \in D}\underbrace{\Big|\min(1-\alpha,1-\frac{m}{n}) (\tilde{K}(\mathbf{x})  - \tilde{K}_{n-m+k}(\mathbf{x}))\Big|}_{\text{scaled disclosed subdomain error}}\\
    &\qquad + \underbrace{2\Big|\alpha - \frac{m}{n}\Big|}_{\text{shifting and scaling errors}}
\end{align*}}}
\end{lemma}

\begin{proof}
Given any $\mathbf{x} = (x_1, x_2)$, we can always find a line passing through the point $(x_1, x_2)$ with the same $\mathbf{W}$ as the decision line. Moreover, the area below the line is greater than that measured by the standard CDF. Since our focus is on $ \sup_{\mathbf{x}}|F(\mathbf{x})-F_{n+k}(\mathbf{x})|$, the maximum differences between empirical and theoretical CDFs can be upper bounded in a larger area (e.g., the area below the line). Therefore, we can obtain the following inequality:
\[\sup_{\mathbf{x} \in D}|F(\mathbf{x})-F_{n+k}(\mathbf{x})| \leq \sup_{\mathbf{x} \in D}|\tilde{F}(\mathbf{x})-\tilde{F}_{n+k}(\mathbf{x})|\]
The equality holds when the maximum difference occurs in the area of ($X_1 \leq x_1, X_2 \leq x_2$). The remainder of the proof follows that of Lemma ~\ref{lemma:right}.
\end{proof}

Now, we present our theorem, extending the multivariate DKW inequality to problems with censored feedback. Analogous to Theorem~\ref{thm:two_subdomains}, we derive the following based on Lemma ~\ref{lemma:multivariate_left} and \ref{lemma:multivariate_right}:

\begin{thm}\label{thm:multivariate_two_subdomains}
Let $\textbf{X}_1, \textbf{X}_2, \ldots, \textbf{X}_n$ be initial/historical IID data samples with cumulative distribution function $F(\textbf{X})$. Let $\mathbf{W}^T\mathbf{X}-b=0$ partition the data domain into two regions, such that $\alpha=F(\mathbf{W}^T\mathbf{X}-b < 0)$, and $m$ of the initial $n$ samples are located in the censored region. Assume we collect $k$ additional samples from the disclosed region, and let $F_{n+k}$ denote the empirical CDF estimated from these $n+k$ (non-IID) data. 
Then, for every $\eta >0$, 
{\begin{align*}
     \mathbb{P}\bigg[\sup_{\textbf{x}\in \mathbb{R}^2} \Big|F(\textbf{x}) - F_{n+k}(\textbf{x})\Big|  \geq \eta \bigg] &  \leq \underbrace{4\exp\Big({\tfrac{-2m(\eta-|\alpha - \frac{m}{n}|)^2}{\min\big(\alpha, \frac{m}{n}\big)^2}}\Big)}_{\text{censored region error (constant)}}\\
     &\qquad + \underbrace{4\exp\Big({\tfrac{-2(n-m+k)(\eta-2|\alpha - \frac{m}{n}|)^2}{\min\big(1-\alpha,\frac{n-m}{n}\big)^2}}\Big)}_{\text{disclosed region error (decreasing with additional data)}}~
\end{align*}
}
\end{thm}
Similarly, incorporating the exploration, the $LB$ can now be defined as a new line sharing the same $\mathbf{W}$ as the decision line but with a different intercept (e.g., $b_{\text{LB}} < b$). Consequently, we have the following:
\begin{thm}\label{thm:multivariate_three_subdomains}
Let $\textbf{X}_1, \textbf{X}_2, \ldots, \textbf{X}_n$ be initial/historical IID data samples with cumulative distribution function $F(\textbf{X})$. Let $\mathbf{W}^T\mathbf{X}-b_{\text{LB}}=0$ and $\mathbf{W}^T\mathbf{X}-b=0$ partition the data domain into three regions, such that $\beta=F(\mathbf{W}^T\mathbf{X}-b_{\text{LB}} < 0)$ and $\alpha=F(\mathbf{W}^T\mathbf{X}-b < 0)$, with $l$ and $m$ of the initial $n$ samples located in the censored region, and the combined censored and exploration region, respectively. Assume we collect an additional $k_e$ samples from the exploration region, under an exploration probability $\epsilon$, and an additional number of $k_d$ samples from the disclosed region. Let $F_{n+k_e+k_d}$ denote the empirical CDF estimated from these $n+k_e+k_d$ non-IID samples. Then, for every $\eta >0$, 
{
\begin{align*}
    &\mathbb{P}\bigg[\sup_{\textbf{x}\in \mathbb{R}^2} \Big|F(\textbf{x}) - F_{n+k_e+k_d}(\textbf{x})\Big|\geq \eta \bigg] \leq 4\exp\Big({\tfrac{-2l(\eta-|\beta - \frac{l}{n} |)^2}{\min\big(\beta, \frac{l}{n}\big)^2}}\Big) \\
    &\qquad \hspace{1.5in} + 4\exp\Big({\tfrac{-2(m-l+k_e)\big(\eta-|\beta - \frac{l}{n}| - \big|\alpha -\beta - \frac{n-l}{n}\frac{m-l+k_e}{n-l+k_e+\epsilon k_d}\big|\big)^2}{\min\big(\alpha-\beta,\frac{n-l}{n}\frac{m-l+k_e}{n-l+k_e+\epsilon k_d}\big)^2}}\Big) \\
    &\qquad \hspace{1.5in} + 4\exp\Big({\tfrac{-2(n-m+k_d)\big(\eta-2\big|\alpha - \frac{l}{n} -  \frac{n-l}{n}\frac{m-l+k_e}{n-l+k_e+\epsilon k_d}\big|\big)^2}{\min\big(1-\alpha,\frac{n-l}{n}\frac{n-m+\epsilon k_d}{n-l+k_e+\epsilon k_d}\big)^2}}\Big).
\end{align*}
}
\end{thm}
The proofs for Theorem~\ref{thm:multivariate_two_subdomains} and ~\ref{thm:multivariate_three_subdomains} follow the same logic as the proofs for Theorem~\ref{thm:two_subdomains} and~\ref{thm:three_subdomains}, with an application of the multivariate DKW inequality and the re-estimated $\alpha$.

\section{Additional experiments} \label{app:more-experiments}

\textbf{Additional experiments on CDF error bounds}

\begin{figure}[ht]
	\centering
    \includegraphics[width=0.5\textwidth]{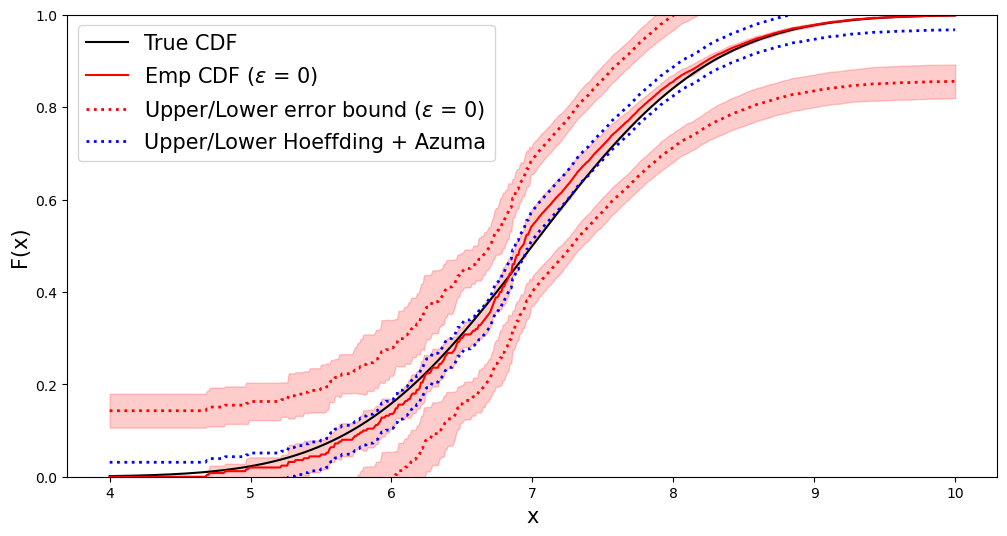}
    \caption{Comparison with our CDF error bound and existing bound.}
	\label{fig:DKW_appendix}
\end{figure}

In this experiment, we aim to compare the performance of our bounds to the existing `Hoeffding + Azuma' bounds derived from Hoeffding and Azuma inequalities \citep{hoeffding1994probability,azuma1967weighted} through the 
confidence bounds ($\delta = 0.015$) around the empirical distribution function. We proceed with the same experiment setting on the synthetic data as in Fig~\ref{fig:DKW}, but without exploration. We run the experiments 5 times and report the average results with its error bars. 

From Figure~\ref{fig:DKW_appendix}, we can clearly see that the existing generalization bound (blue dotted lines) fails to capture the true distribution, especially in the censored and exploration regions, from where we do not collect samples due to the censored feedback. This is because the discrepancy between the empirical and true distributions is large in regions with fewer samples, and the existing generalization bound does not account for censored feedback, decreasing the error bounds faster. 

\textbf{Comparison of \emph{a priori} and average of \emph{a posteriori} bounds.} The Fig.~\ref{fig:compare-apriori-posterior} shows that the average of \emph{a posterior} bounds over multiple runs can be viewed as an approximation of the \emph{a priori} bounds. 
\begin{figure}[ht]
	\centering
    \includegraphics[width=0.5\textwidth]{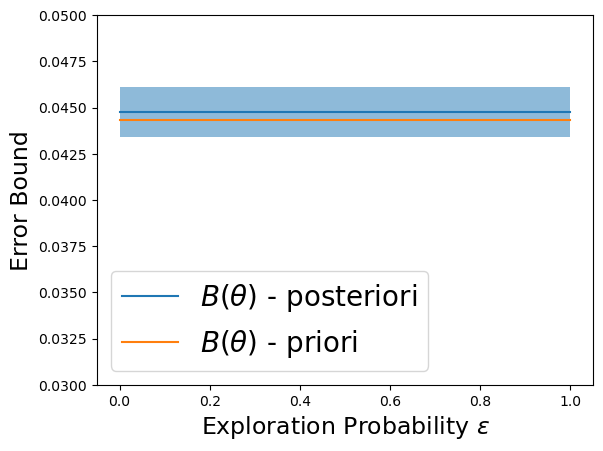}
    \caption{Comparison of Corollary~\ref{cor:a-priori-bound} and average of Theorem~\ref{thm:two_subdomains} bounds.}
	\label{fig:compare-apriori-posterior}
\end{figure} 

\end{appendices}

\section*{Declarations} 
\begin{itemize}
\item Funding: This work is supported in part by the National Science Foundation (NSF) program on Fairness in AI in collaboration with Amazon under grant IIS-2040800, NSF grant CMMI-2411000, and Cisco Research. Any opinion, findings, and conclusions or recommendations expressed in this material are those of the authors and do not necessarily reflect the views of the NSF, Amazon, or Cisco. 
\item Conflict of interest: Not applicable.
\item Ethics approval and consent to participate: Not applicable.
\item Consent for publication: Not applicable.
\item Data availability: All datasets are publicly available.
\item Materials availability: Not applicable.
\item Code availability: The code is included in the supplementary material. 
\item Author contribution: Contributing authors are Yifan Yang, Ali Payani, and Parinaz Naghizadeh. 
\end{itemize}

\bibliography{sn-bibliography}

\end{document}